%% file: main.tex
\documentclass{article} 
\usepackage{iclr2025_conference,times}
\usepackage{wrapfig}

\usepackage{microtype}
\usepackage{graphicx,subcaption}
\usepackage{subcaption}

\usepackage{amsmath}
\usepackage{amssymb}
\usepackage{mathtools}
\usepackage{amsthm}
\usepackage{url}  
\usepackage[hidelinks,colorlinks=true,linkcolor=blue,citecolor=blue]{hyperref}

\usepackage{tikz}
\newcommand{\cblock}[3]{
  \hspace{-1.5mm}
  \begin{tikzpicture}
    [
    node/.style={square, minimum size=10mm, thick, line width=0pt},
    ]
    \node[fill={rgb,255:red,#1;green,#2;blue,#3}] () [] {};
  \end{tikzpicture}%
}
\definecolor{Yellow}{rgb}{1.0, 1.0, 0.0}
\definecolor{Gray}{rgb}{0.5, 0.5, 0.5}
\definecolor{gg}{HTML}{32932f}
\definecolor{stabilo}{HTML}{90EE90}
\definecolor{purple}{HTML}{C3B1E1}

\usepackage[capitalize,noabbrev]{cleveref}

\theoremstyle{plain}
\newtheorem{theorem}{Theorem}[section]

\newtheorem{lemma}[theorem]{Lemma}
\newtheorem{corollary}[theorem]{Corollary}
\theoremstyle{definition}
\newtheorem{definition}[theorem]{Definition}

\theoremstyle{remark}

\usepackage{enumitem}

\input{math_commands.tex}

\title{Mixture of Parrots \includegraphics[width=1em]{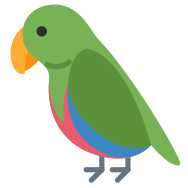}\includegraphics[width=1em]{images/parrot.png}\includegraphics[width=1em]{images/parrot.png}: Experts improve \\ memorization more than reasoning}

\author{Samy Jelassi\thanks{Correspondence to: Samy Jelassi $<$\url{sjelassi@fas.harvard.edu}$>$} \\
Harvard University\\
\And
Clara Mohri \\
Harvard University\\
\And
David Brandfonbrener\\ Harvard University\\Kempner Institute
\And Alex Gu\\
MIT
\And Nikhil Vyas\\Harvard University
\And Nikhil Anand\\Harvard University\\Kempner Institute
\And David Alvarez-Melis\\Harvard University\\Kempner Institute
\And Yuanzhi Li\\Microsoft Research
\And Sham M. Kakade\\Harvard University\\Kempner Institute
\And Eran Malach\\Harvard University\\Kempner Institute
}

\iclrfinalcopy 
\begin{document}

\maketitle

\begin{abstract}
\input{abstract}

\end{abstract}

\input{introduction}
\input{related_work}

\input{theory}

\input{synthetic}

\input{pretrain}

\input{discussion}

\input{acknowledgements}


\bibliography{references}
\bibliographystyle{iclr2025_conference}
\vfill
\pagebreak
\appendix

\input{app_related_work}

\input{app_experiments}
\input{app_proofs}

\end{document}

%% file: math_commands.tex
\def\Nset{\mathbb{N}}
\def\Rset{\mathbb{R}}

\def\node{{\mathsf n}}

\DeclarePairedDelimiter{\abs}{\lvert}{\rvert} 
\DeclarePairedDelimiter{\bracket}{[}{]}
\DeclarePairedDelimiter{\curl}{\{}{\}}
\DeclarePairedDelimiter{\paren}{(}{)}
\DeclarePairedDelimiter{\norm}{\|}{\|}

\DeclareMathOperator*{\E}{\mathbb{E}}
\DeclareMathOperator*{\argmax}{\rm argmax}

\let\Pr\relax 
\DeclareMathOperator*{\Pr}{\mathbb{P}}

\providecommand{\abs}[1]{\lvert#1\rvert}
\providecommand{\norm}[1]{\lVert#1\rVert}

\newcommand{\cH}{\mathcal{H}}

\newcommand{\cN}{\mathcal{N}}

\newcommand{\cX}{\mathcal{X}}

\newcommand{\mat}[1]{{\mathbf #1}}

\newcommand{\br}{{\mathbf r}}
\newcommand{\bu}{{\mathbf u}}
\newcommand{\bv}{{\mathbf v}}

\newcommand{\I}{\mat{I}}

\renewcommand{\a}{\mat{a}}

\newcommand{\1}{\mat{1}}

\newcommand{\ignore}[1]{}

\newenvironment{proof*}{\trivlist\item[\hskip\labelsep{\it Proof}{.}]}

\usepackage{amsmath,amsfonts,bm}

\def\eqref#1{equation~\ref{#1}}

\def\1{\bm{1}}

\def\vb{{\bm{b}}}

\def\vu{{\bm{u}}}

\def\vx{{\bm{x}}}

\def\mW{{\bm{W}}}

\DeclareMathAlphabet{\mathsfit}{\encodingdefault}{\sfdefault}{m}{sl}
\SetMathAlphabet{\mathsfit}{bold}{\encodingdefault}{\sfdefault}{bx}{n}

\newcommand{\softmax}{\mathrm{softmax}}

\newcommand{\BOS}{\left\langle{\tiny\textsc{BOS}}\right\rangle}
\newcommand{\SEP}{\left\langle{\tiny\textsc{SEP}}\right\rangle}
\newcommand{\EOS}{\left\langle{\tiny\textsc{EOS}}\right\rangle}

\newcommand{\PAD}{\left\langle{\tiny\textsc{PAD}}\right\rangle}

\newcommand{\EDGE}{\left\langle{\tiny\textsc{EDGE}}\right\rangle}

%% file: abstract.tex
The Mixture-of-Experts (MoE) architecture enables a significant increase in the total number of model parameters with minimal computational overhead. 
However, it is not clear what performance tradeoffs, if any, exist between MoEs and standard dense transformers.
In this paper, 
we show that as we increase the number of experts (while fixing the number of active parameters), the memorization performance consistently increases while the reasoning capabilities saturate. We begin by analyzing the theoretical limitations of MoEs at reasoning. We prove that there exist graph  problems that cannot be solved by any number of experts of a certain width; however, the same task can be easily solved by a dense model with a slightly larger width. 
On the other hand, we find that on memory-intensive tasks, MoEs can effectively leverage a small number of active parameters with a large number of experts to memorize the data. 
We empirically validate these findings on synthetic graph problems and memory-intensive closed book retrieval tasks. 
Lastly, we  pre-train a series of MoEs and dense transformers and evaluate them on commonly used benchmarks in math and natural language. 
We find that increasing the number of experts helps solve knowledge-intensive tasks, but fails to yield the same benefits for reasoning tasks.

%% file: introduction.tex
\section{Introduction}

The explosion in capabilities of large language models in recent years has largely been enabled by scaling their size, as measured by the number of parameters in the model. In the standard Transformer architecture, scaling the number of parameters entails a proportional increase in computational cost, e.g. doubling the number of parameters requires doubling the number of floating-point operations (FLOPs), making training and inference more computational intensive. Mixture-of-Experts (MoE) were introduced  as a solution for this problem \citep{shazeer2017outrageously,lepikhin2020gshard,fedus2022switch}. MoEs replace the single MLP in each Transformer block with multiple MLPs (called experts), where each token is routed to a few experts based on a linear routing function. The number of parameters in the MoE layer therefore increases with the total number of experts, while the compute increases only with the number of ``active'' experts (i.e., the number of experts to which the token is routed to). This offers a promising option for scaling models: increase the number of experts instead of the model dimension or its depth. For this reason, MoEs have become very popular, and many frontier models today are based on the MoE architecture \citep{achiam2023gpt,databricks2023,anil2023gemini,dai2024deepseekmoe,jiang2024mixtral,yang2024qwen2}.

In this work we study whether MoE indeed offers a ``free-lunch'', enabling gains in performance with no computational cost. Interestingly, we find that the benefit from MoEs greatly depends on the task at hand. We show that for reasoning-based tasks, such as graph problems and mathematical reasoning, MoEs offer limited performance gains, and increasing the number of experts cannot compete with scaling the dimension (width) of the model. On the other hand, for memory-intensive tasks, we show that scaling the number of experts is competitive with scaling standard ``dense'' MLPs.

\begin{figure*}[t]
\centering
\begin{subfigure}[t]{0.31\textwidth}
\centering
\vskip 0pt
    \includegraphics[height=3.cm]{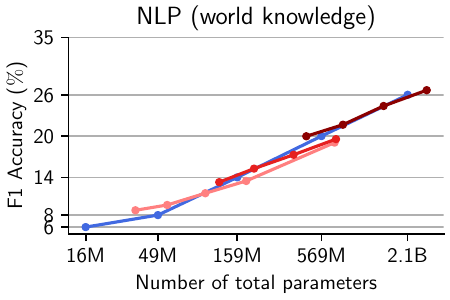}
   \scriptsize
    \vspace*{-.2cm}
    
        \mbox{\hspace*{1.cm}\cblock{65}{105}{225}\hspace{1mm}Dense transformer \hspace{1.5mm} \cblock{255}{128}{ 128}\hspace{1mm}MoE (18M active parameters) \hspace{1.5mm} \cblock{230}{32}{ 32}\hspace{1mm}MoE (58M active parameters)\hspace{1.5mm}  \cblock{139}{0}{ 0}\hspace{1mm}MoE (200M active parameters)}
    \vspace{.0cm}
    \caption{Evaluation: world knowledge}
    \label{fig:world_knowledge_intro}
\end{subfigure}
\hfill
\begin{subfigure}[t]{0.31\textwidth}
\centering
\vskip 0pt
\includegraphics[height=3.cm]{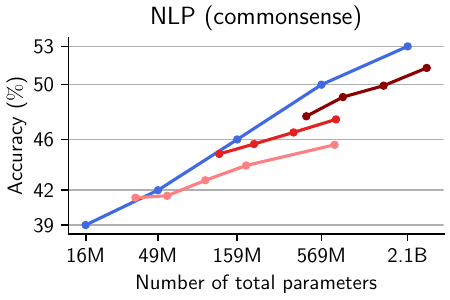}

\vspace*{.25cm}

\hspace*{1.5cm}\caption{Evaluation: commonsense}
\label{fig:commonsense_intro}
\end{subfigure}
\hfill
\begin{subfigure}[t]
{0.31\textwidth}
\centering
\vskip 0pt
\includegraphics[height=3.cm]{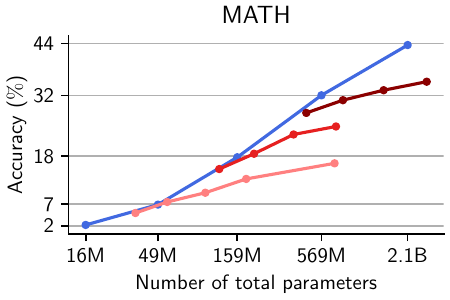}
\vspace{.25cm}
\caption{Evaluation: math}
\label{fig:math_intro}
\end{subfigure}
\label{fig:main}
\caption{\textbf{(a) Evaluation: world knowledge.} We train a series of dense transformers and MoEs on 65B tokens from a corpus essentially made of Fineweb-edu, Cosmopedia and Wikipedia (see \autoref{sec:pretrain} for details). We then evaluate the models on several world knowledge benchmarks (e.g., TriviaQA \citep{joshi2017triviaqa}, Natural Questions \citep{kwiatkowski2019natural}) and report the average F1 accuracy. Surprisingly, at a fixed number of total parameters, MoEs with substantially fewer active parameters approximately match the performance of dense models. This highlights the importance of experts in tasks that require memorization.  
\textbf{(b) Evaluation: commonsense.} Here we evaluate the aforementioned  pre-trained models on natural language commonsense benchmarks (e.g., HellaSwag \citep{zellers2019hellaswag}, WinoGrande \citep{sakaguchi2021winogrande}). On these reasoning tasks, we observe that MoEs perform worse than dense models and more significant benefits are obtained by increasing the number of active parameters. 
\textbf{(c) Evaluation: math.} Here we train a series of dense transformers and MoEs on 65B tokens from a corpus essentially made of Proof-Pile2
\citep{azerbayev2023llemma} (see \autoref{sec:pretrain} for details). The results are consistent with the ones in (b): MoEs perform worse than dense models at equal number of total parameters.}\label{fig:introduction}

\vspace*{-1.5cm}

\end{figure*}

To demonstrate these claims, we begin with a theoretical analysis of MoEs and dense models. 
We use communication-complexity lower bounds to show that a single-layer MoE requires a critical dimension to solve a simple graph connectivity problem, implying that MoEs offer no benefit for solving this problem and only consume unnecessary memory. On the other hand, we show that for a pure memorization task, where the model only needs to ``remember'' an arbitrary set of examples, scaling the number of experts is equivalent to scaling the number of parameters in dense transformers, implying a significant computational gain when fixing the number of active parameters (\autoref{sec:theory}). We continue by experimentally validating these results, comparing MoEs against dense models on synthetic training data. We train these models on finding the shortest path in random graphs, where we show that MoE accuracy does not improve as we increase the number of experts, but that accuracy consistently increases with width for a dense transformer (\autoref{fig:shortest_path_n50_total_params}).
Following this, we train different models on the task of memorizing a large phone-book. We demonstrate that MoEs excel in memorization, matching the performance of dense transformers with the same number of total parameters but with substantially less computational cost (\autoref{fig:phonebook_total_params}).

Finally, we train dense transformers and MoEs on real datasets of mathematical reasoning and natural language, and perform intensive benchmarking of these models on a wide variety of downstream tasks. For memory-intensive tasks, MoEs surprisingly have a great advantage, where increasing the number of experts can match the performance of large dense models (\autoref{fig:world_knowledge_intro}). However, we show that for tasks that rely on reasoning, scaling the number of experts cannot compete with increasing the model dimension (Figures \ref{fig:commonsense_intro}-\ref{fig:math_intro}). Moreover, MoEs exhibit some memorization behaviors when trained on math problems (\autoref{fig:generalization_gap}).   Taken together, our results show that the gains from using MoEs depend greatly on the nature of the training data and downstream task, and that while MoEs can improve performance in certain cases, sometimes increasing the effective size (width) of the model is unavoidable. 

%% file: related_work.tex
\section{Related work}

\paragraph{Mixture of Experts.} Mixture-of-Experts (MoE) date back to the work of \cite{jacobs1991adaptive,jordan1994hierarchical}. \cite{shazeer2017outrageously,fedus2022switch} were the first to scale this idea to deep learning and obtain state-of-the-art models in machine translation. Since then, several works have improved their routing algorithms \citep{lepikhin2020gshard,lewis2021base,roller2021hash,clark2022unified,zhou2022mixture,antoniak2023mixture,zhong2024lory}, have improved their downstream performance after finetuning \citep{du2022glam,zoph2022st} or made their training and inference more efficient \citep{rajbhandari2022deepspeed,gale2023megablocks,pan2024dense,tan2024scattered}. However, only a few papers have studied the science of MoEs and their comparison with dense transformers. \cite{clark2022unified,krajewski2024scaling} establish scaling laws for MoEs. \cite{chen2022towards} design a specific classification problem where a model with multiple experts provably outperforms one with only one expert. \cite{shazeer2017outrageously,lepikhin2020gshard,artetxe2021efficient,lewis2021base,fedus2022switch,du2022glam} show that given a fixed FLOP budget, MoEs are always better. However, these papers claim that on a per parameter basis, MoEs always seem comparatively worse than dense models. In this paper, we temper this claim by showing that it depends on the \emph{nature of the task} at hand: on reasoning tasks, we validate this claim but on memory-intensive tasks, equally-sized MoEs perform as well as dense transformers.

\paragraph{Language models and memorization.} Large language models (LLMs) store a considerable amount of knowledge in their parameters \citep{petroni2019language,heinzerling2020language}. 
They memorize useful knowledge such as facts and commonsense \citep{zhao2024large}.
Many works studied how memorization occurs in LLMs by developing tools to locate the knowledge in the model \citep{meng2022locating,allen2023physics,liu2024devil} or by tracking the training dynamics \citep{tirumala2022memorization,speicher2024understanding}. We draw inspiration from \cite{allen2023physics} and evaluate the memorization of our models by pre-training them on a mixture of datasets that includes Wikipedia, and at test time, evaluate them on world knowledge benchmarks, which are essentially question answering tasks on Wikipedia facts. 
With respect to theoretical findings, \cite{kim2023provable,mahdavi2023memorization,madden2024upper,nichani2024understanding}  provide upper bounds on the number of parameters needed for dense transformers to perform memorization tasks under various conditions.

\paragraph{Language models and reasoning.} In recent years, transformer-based language models have displayed remarkable effectiveness in solving a broad range of reasoning tasks. Specifically, the reasoning capabilities of transformers have been studied in the context of arithmetic problems \citep{jelassi2023length,cho2024position,hou2024universal,zhou2024transformers, mcleish2024transformersarithmeticrightembeddings,lee2023teaching},  mathematical reasoning \citep{zhang2022unveiling, imani2023mathprompter,wei2022chain} graph problems \citep{sanford2024understanding,fatemi2023talk,jin2023large,wang2024can} and code challenges \citep{shi2024can,zhu2024deepseek}. Recently, state-of-the-art language models were used for solving complex math olympiad problems \citep{deepmind2024ai,numina2024winning,openai2024o1}. 
With respect to theoretical findings, various works study the reasoning capabilities of transformers, relating their expressive power to other complexity classes and formal languages \citep{DBLP:conf/icml/WeissGY21, zhou2023algorithmstransformerslearnstudy,strobl2024formal,chen2024can}. Other works study how chain-of-thought can improve the reasoning capabilities of language models in terms of expressive power and learnability \citep{abbe2024fartransformersreasonlocality,merrill2023expresssive,malach2023auto}. 

%% file: theory.tex
\section{Theory: representational capacity}\label{sec:theory}

In this section, we analyze the capability of MoE transformers compared to standard (dense) models. We begin by studying a simple graph problem that requires scaling the hidden dimension of the transformer, showing that MoEs with small hidden dimension cannot solve this problem, regardless of the number of experts used. Then, we show that MoEs can effectively memorize random inputs, requiring significantly less computational resources (active parameters) compared to dense models.

\subsection{Setting}
Consider a one-layer transformer $f \in \text{Transformer}_{m, H, 1}^N$ which takes as input a sequence of length $N$ and has logarithmic bit-precision. $f$ embeds the input into dimension $m$ via the function $\phi$. $f$ has $H \geq 1$ attention heads, whose outputs are combined via concatenation before we apply point-wise function $\psi$ \footnote{In multi-layer Transformers, each layer outputs a vector of size $m$. However, since our focus in this section will be on binary classification problems, we will let the transformer output a single scalar, and we interpret the output of the final token as the prediction for the classification task.}. 

Define the parameters as $Q_h, V_h, K_h \in \Rset^{m \times m}$, $\phi: \cX \rightarrow \Rset^m$, $\psi: \Rset^m \rightarrow \Rset$.  
The output of $f$ is:
\[
f(\vx_1, \dots, \vx_N) = \psi \paren*{\bracket*{\softmax \paren*{\phi(x_N)^\top Q_h K_h^\top \phi(X)}\phi(X) V_h }_{h \in [H]} }.
\]
$f$ is a \emph{dense} transformer, if $\psi$ is an MLP, i.e. function of the form:
\[
\psi(\vx) = \vu^\top \sigma(\mW \vx + \vb)\mathrm{,~~for~} \mW \in \Rset^{m' \times m}, b \in \Rset^{m'}, \vu \in \Rset^{m'}
\]
where $\sigma$ is the ReLU activation function. 
$f \in \text{Transformer}_{m, H, 1, K}^N$ is a \emph{sparse} (MoE) transformer with $K$ experts if $\psi$ is a function of the form:
\[
\psi(\vx) = \bu_{i}^{\top} \sigma(\mW_i \vx + \vb_i)\mathrm{~for~}i=\argmax_{j} \br_j^\top \vx
\]
where $\mW_1, \dots, \mW_k \in \Rset^{m' \times m}$, $\vb_1, \dots, \vb_k \in \Rset^{m'}$, $\bu_1, \dots, \bu_k \in \Rset^{m'}$ are the parameters of each expert and $r_1, \dots, r_k$ define the routing function (we use top-1 routing).

\subsection{MoEs require a critical hidden size to solve graph reasoning tasks}
In this section, we analyze the graph reasoning capabilities of dense and sparse transformers. We define the length-2 path problem on a graph, and use it as a means to understand other graph reasoning tasks such as graph connectivity, shortest path, and cycle detection. 

\begin{definition}[Length-2 Path Problem]
    The input is a graph $G = (V, E)$. The source  $s \in V$ and a destination $d \in V$ are fixed for all tasks as the $0$ and $\abs{V}$ vertex. The length-2 path problem asks whether there is a path of length $2$ from $s$ to $d$. 
\end{definition}
Graph connectivity, shortest path, and cycle detection are all graph reasoning tasks which reduce to the length-2 path problem due to \citep{sanford2024understanding} and Lemma \ref{lemma:SD_L2path}.
We provide a lower-bound on the width required for a sparse transformer to solve the length-2 path problem, and an upper-bound on the width required for a dense transformer to solve the problem. Further, we show a separation between dense and sparse transformers with the same number of parameters: for a sufficiently large amount of experts in the sparse model, it cannot solve the same problem that a dense model can solve with the same amount of \emph{total} parameters.

\paragraph{Lower bound on width of depth-1 MoE for reasoning.}
\label{sssec:lb_width_moe}
We begin by showing a lower-bound on the width for a depth-1 mixture of expert model for the length-2 path problem. This lower bound implies a lower bound for search and retrieval tasks such as graph connectivity, shortest path, and cycle detection. 
\begin{theorem}[Length-2 path lower-bound on sparse transformers]
\label{thm:cc_lowerbound}
    For some input sequence $G = (V, E)$, fix two disjoint subsets $A, B \subset [N-1]$, and consider a single-layer transformer $f \in \text{Transformer}_{m, H, 1, K}^N$   with $O(\log N)$-bit precision that solves length-2 path for any input $X$ where $X_A$ is a function of edges with the source $s$, $X_B$ is a function of edges with the destination $d$. Then, $f$ has width satisfying $mH = \Omega(\abs{V} / \log N)$.
\end{theorem}
The proof follows almost identically from the proof in \citep{sanford2024understanding} for the class $\text{Transformer}_{m, H, 1}^N$. The original proof does not place constraints on the function $\psi$ and is based on a communication-complexity argument. As such we may design $\psi$ so that it first routes and then chooses which expert to apply. We give a complete proof in Appendix \ref{app:proofs}.  As such, the result of \citep{sanford2024understanding} can also be extended to the class $\text{Transformer}_{m, H, 1, K}^N$.
\paragraph{Upper bound on width of depth-1 dense transformer for reasoning.}
\label{sssec:ub_width_dense}
In this section we give an upper bound for the width required for a dense model to solve the length-2 path problem.

\begin{theorem}[Length-2 path width upper bound for transformer]
\label{thm:mem_ub}
There exists a transformer of width $\abs{V}$, $H = 1$, and $O(\log N)$-bit precision that solves length-2 path problem for any input.
\end{theorem}
The proof relies on an encoding of the inputs where the output values only exceed a certain threshold when $u$ and $v$, the source and destination vertices, have edges with a common vertex. We defer the proof to Appendix \ref{app:proofs}. 

\paragraph{Parameter-matched comparison of dense and sparse depth-1 transformers.}
Using the lower-bound on width required for a sparse transformer (Theorem \ref{thm:cc_lowerbound}) and the upper-bound on width required for a dense transformer (Theorem \ref{thm:mem_ub}), we compare dense and sparse transformers when they have the same number of total parameters. We find that when the number of experts exceeds $(\log N)^2$, the sparse model is unable to solve the same task as the dense model.
\begin{corollary}
    Consider a sparse transformer (with $K$ experts) and a dense transformer with the same number of parameters. There exists a number of experts $K$ so that the the sparse model is not able to solve the reasoning task, but the dense transformer solves the task.
\end{corollary}
\begin{proof}
Suppose we have two depth-$1$ transformers, where one is a dense model and the other is a mixture of experts with $K$ experts. Let the width of the dense model be $m_d$, and the width of the sparse model be $m_s$. The number of parameters in the dense model is $O(m_d^2)$ and the number of parameters in the sparse model is $O(Km_s^2)$. In order to match the number of parameters, it must be the case that $m_s = \frac{m_d}{\sqrt{K}}$.
Suppose we let $m_d = \abs{V}$, as this is sufficient to solve the above problems. For any $K \geq \Omega \paren*{(\log N)^2}$, the sparse model is not sufficiently wide to solve the problem.     
\end{proof}

\subsection{MoEs use their experts to solve memory-intensive tasks}


In this section, we provide an upper-bound on the number of parameters necessary for a sparse transformer to solve memorization tasks, followed by a lower-bound on the number of parameters needed for a dense transformer to solve the same task. We use these results to compare the memorization capabilities of dense and sparse transformers with the same number of active parameters. We find that with enough experts, the sparse transformer is able to solve memorization tasks with less active parameters than the dense transformer. In both bounds we assume that transformer has logarithmic number of bits to encode each parameter.

We consider  sequences $\{(X^i, y_i)\}_{i=1}^n$ where $X^i \in \Rset^{N \times m}$ are input sequences of length $N$ in dimension $m$ such that $X^i[j]$ is sampled from a Gaussian distribution $\mathcal{N}(0,I_m)$. We assume $y_1, \dots, y_N \in \{\pm \1\}$ are arbitrary labels for the $n$ sequences. The objective is for a transformer to memorize these sequences, i.e. map each input $X^i$ to a label $y_i$. The classification is determined by the sign of the last token output.

\paragraph{Upper-bound on  MoE for memorization.}
We begin by showing that, with high probability over the choice of the inputs, the MoE architecture can memorize (i.e., arbitrarily label the examples), with a small number of active parameters.

\begin{theorem}\label{thm:moe_transformer}
With probability at least $0.99$, there exists a one-layer MoE transformer with \( K \) experts, using \( \tilde O\left( \dfrac{n}{K} +Km\right) \)  active parameters and \(\tilde O\left(n+Km\right)\) total parameters stored in $\tilde{O}(1)$ bits that, when applied to each sequence \( X^i \), outputs at the last token a value whose sign matches \( y_i \), i.e.,
$\operatorname{sign}(f(X_i)) = y_i \quad \text{for all } i = 1, \dots, n.$\footnote{We use $\tilde{O}$ and $\tilde{\Omega}$ to hide logarithmic factors.}
\end{theorem}

Specifically, if we choose $K = \sqrt{n/m}$ we get that an MoE architecture can solve the memorization problem with $\tilde O(\sqrt{nm})$ parameters. To prove this, we show that for a particular routing function, the number of samples routed to each expert is approximately $n/K$. Then, we show that an expert with $\tilde{O}(n/{mK})$ neurons can memorize a sample of size $O(n/K)$. We present the proof in Appendix \ref{appc:memorization}.

\paragraph{Lower bound on memorization with dense Transformer.}
Next, we give a lower-bound on the number of parameters for a dense transformer to perform memorization.
\begin{theorem}[Lower bound for dense model]
    \label{thm:lower_bound_memorization}
    Given the same task as above, a dense Transformer requires $\tilde{\Omega}(n)$ parameters to solve the memorization task.
\end{theorem}
This bound follows from the fact that there are $2^n$ possible labels for any fixed set of $n$ inputs, and at most $2^{cW}$ functions with $W$ parameters and $c$ bit per parameters. The proof is in Appendix \ref{appc:memorization}.

\paragraph{Separation between MoEs and Dense Models.}
Observe that the previous results on memorization imply a separation between MoEs and dense models in terms of the number of active parameters. Namely, we show that an MoE with $\tilde O(\sqrt{nm})$ active parameters can memorize, while a dense model requires $\tilde{\Omega}(n)$ parameters. So, for $n \gg m$, MoEs are significantly more efficient. Comparing the number of total parameters, MoEs require $\tilde O(n + Km)$ parameters, so both MoE and dense models have linear dependence on $n$ in the total parameter count.

%% file: synthetic.tex
\section{Synthetic experiments}\label{sec:synthetic}

In the previous section, we proved that there exist graph connectivity problems that cannot be solved by any number of experts of a certain width but the same task can be solved by a dense model with a slightly larger width. Our goal in this section is to verify that our theoretical analysis bears out experimentally when training models from scratch on synthetic data, before moving on to study pre-trained models in \autoref{sec:pretrain}. We mainly focus on two tasks: the \emph{shortest path} problem (\autoref{fig:shortest_path}), which we use as a synthetic task to represent reasoning problems, and the \emph{phone-book} task (\autoref{fig:phonebook}), to measure the recall ability of our models. Our experiments in this section highlight that adding experts yields greater performance improvements on memorization tasks than reasoning tasks.

\begin{figure}[h]

\begin{subfigure}[t]{0.75\textwidth}

\includegraphics[width=1.\linewidth]{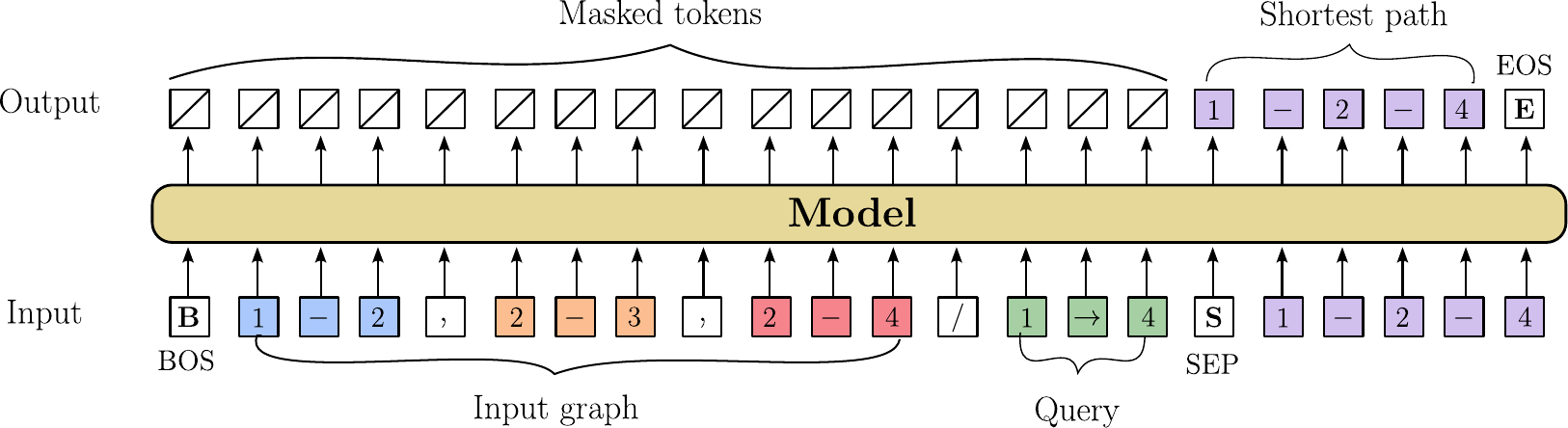}
\end{subfigure}
\begin{subfigure}[t]{0.24\textwidth}

\hspace*{.8cm}\includegraphics[height=2.52cm]{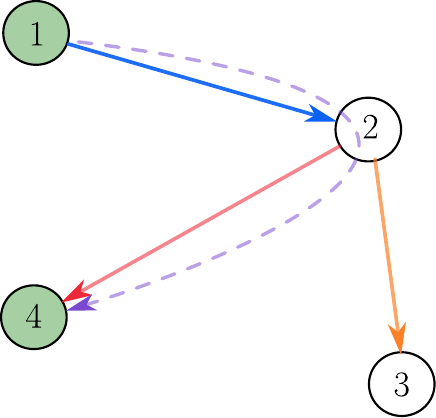}
\end{subfigure}
\vspace{-0.2cm}
\caption{\small Illustration of the shortest path task. We feed the model with a sequence that lists all the edges in the input graph and ends with the query (in green) which asks the model to find a shortest path between two vertices (from vertex 1 to vertex 4 in the figure). The model then autoregressively returns the shortest path (in purple).   }
\vspace{-0.2cm}
\label{fig:shortest_path}
\end{figure}

\subsection{Experimental setup}

\paragraph{Architecture.} We opt for the Mistral \citep{jiang2023mistral} and Mixtral \citep{jiang2024mixtral} architectures as the backbones of our Transformer and MoE models, respectively. The two architectures are identical and only differ by the addition of a gating module and multiple experts in Mixtral. For both model types, we fix the number of layers to $L=12$. For the dense transformers, we vary model size by sweeping the width $d\in\{256,512,1024\}$. For MoEs, we sweep over widths $d\in\{256,512\}$ and the number of experts $E\in\{8,16,32,64\}$. To be consistent with our experiments in \autoref{sec:pretrain}, we set the intermediate dimension in the FFN block to be equal to $d$ (and not $4d$). We use token-choice routing, do not apply any token dropping and each token is routed to the top-2 experts. Lastly, in both this section and \autoref{sec:pretrain}, we report for each model the number of non-embedding parameters which we refer to as the total number of parameters.

\paragraph{Shortest path task.} For a graph with $n$ vertices, our token space $\mathcal{V}$ is of size $n+6$  with tokens encoding the vertices and some special tokens: $\mathcal{V}=\{1,\dots,n,\EDGE,\BOS,\EOS,\PAD,\SEP,/\}$ where $\BOS$ is the beginning of sentence token, $\EOS$ the end of sentence token, $\PAD$ the padding token, $\EDGE$ is the token indicating an edge between two vertices and, $\SEP$ and ``/'' are  separator tokens. Each sequences describes the graph by a list of all the edges followed by two randomly sampled vertices and the shortest path between these latter (see \autoref{fig:shortest_path}). All the graphs are directed and sampled according to the Erd\"os-R\'{e}nyi model, with $n$ vertices and probability $p$ for each edge to exist. We vary $n\in\{25,30,50,40,45,50,55\}$ and set $p$ such that the average length of the shortest path is 3.5. Each train/test pair corresponds to \emph{one} value of $(n,p)$, we do not mix graph sizes.

\begin{figure}[t]

\vspace{-0.5cm}
\begin{subfigure}[t]{0.55\textwidth}
\includegraphics[width=1.\linewidth]{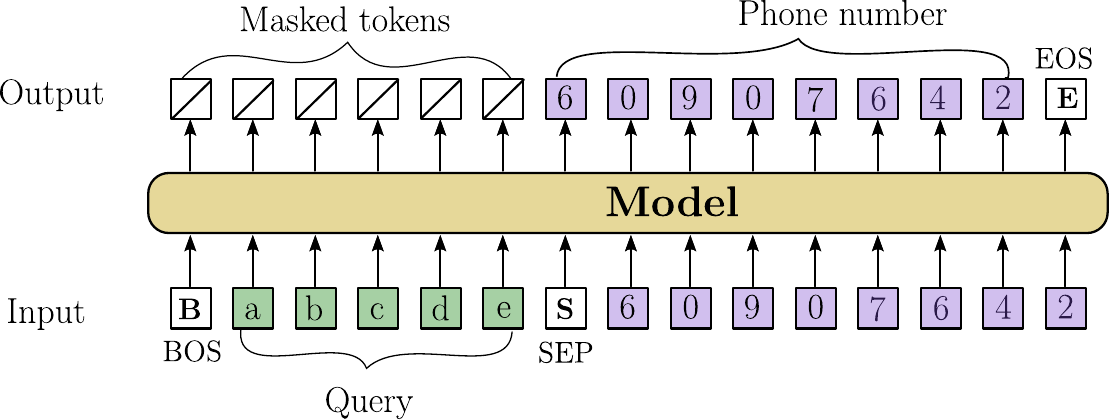}
\end{subfigure}
\hfill
\begin{subfigure}[t]{0.4\textwidth}
\vspace{-3.2cm}
\centering
\includegraphics[height=2.8cm]{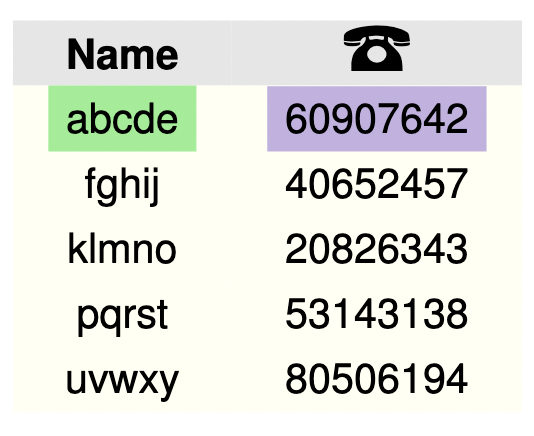}
\end{subfigure}
\vspace{-0.2cm}
\caption{\small Illustration of the phone-book task for closed-book retrieval. The model is first trained to memorize a phone-book (illustrated on the right). Then, we randomly select a name in the phone-book (in green) and ask the model to return their phone number (in purple) without access to the phone-book.}
\vspace{-0.2cm}
\label{fig:phonebook}
\end{figure}

\begin{figure}[t]

\begin{subfigure}[t]{0.40\textwidth}

\centering
\includegraphics[width=1\linewidth]{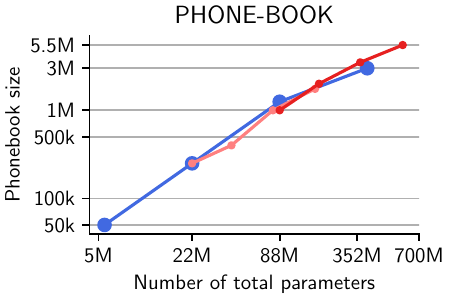}
\small 
\centering
    \vspace*{-.2cm}
        \mbox{\hspace*{1.cm}\cblock{65}{105}{225}\hspace{1mm}Dense transformer \hspace{1.5mm} \cblock{255}{128}{ 128}\hspace{1mm}MoE (10M active parameters) \hspace{1.5mm} \cblock{230}{32}{ 32}\hspace{1mm}MoE (42M active parameters)} 

\vspace{.25cm}

\caption{Phone-book memorization}\label{fig:phonebook_total_params}
\end{subfigure}\hfill
\begin{subfigure}[t]{0.40\textwidth}

\centering

\includegraphics[width=1\linewidth]{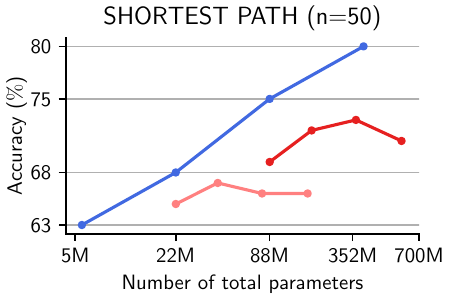}

\vspace{.50cm}

\captionsetup{width=1.2\columnwidth}
\caption{\small Shortest path (50-node graphs)}\label{fig:shortest_path_n50_total_params}
\end{subfigure}
 
\vspace{-0.2cm}
\caption{\small\textbf{(a) Phone-book memorization}: We train a series of dense transformers and MoEs on phone-books of varying sizes and then evaluate their memorization capacity. We report the maximal phone-book size where the model obtains more than 90\% accuracy. The maximal phone-book size correlates with the total (and not active) number of parameters. \textbf{(b) Shortest path (total parameters)}: We train models to find the shortest path in 50-node graphs and report the test accuracy. Here, increasing the number of experts provides limited improvements and the performance rather correlates with the number of active parameters.} 
\vspace{-0.2cm}
\label{fig:synthetic_total_param_exp}
\end{figure}

\paragraph{Phone-book task.}  Our token space $\mathcal{V}$ is of size 39 and made of the alphabet letters, digits and special tokens: $\mathcal{V}=\{a,\dots,z,0,\dots,9,\BOS,\EOS,\SEP\}$. We generate phone-books where the names consist of 5 letters and the phone numbers of 8 digits (see \autoref{fig:phonebook}). We ensure that both the names and numbers are unique.

\paragraph{Datasets.} For the graph experiments, the training set size is $1\mathrm{e}6$ and the test set consists of $1\mathrm{e}3$ held-out examples that are sampled from the \emph{same} distribution as the training examples. For the phone-book experiments, we vary the training set size over $\{1\mathrm{e}{5},5\mathrm{e}{5},1\mathrm{e}{6},1.5\mathrm{e}{6},2\mathrm{e}{6},2.5\mathrm{e}{6},3\mathrm{e}{6}\}$ and the test set consists of $1\mathrm{e}3$ queries from the training set. 

\paragraph{Optimization.} We use the AdamW optimizer \citep{loshchilov2017fixing} with a weight decay equal to 0.1. We sweep the learning rate over $\{5\mathrm{e}{-5},1\mathrm{e}{-4},5\mathrm{e}{-4},1\mathrm{e}{-3}\}$, the number of epochs over $\{2,5,10,15\}$, and set the maximal possible batch size among $\{8,16,32\}$. We use a warmup during the 20\% first  training steps and a linear decay scheduler. All models are trained by next-token prediction. In the graph task, we apply a mask on the input instance so that we only penalize the model whenever it makes a mistake on the labels (and not on the inputs and labels jointly). In the phone-book experiment, we do not apply any masking.

\paragraph{Evaluation.} For each task we compute the exact accuracy, i.e.\ we count the generation as correct only if it fully matches the ground truth. For the phone-book task, we report the size of the maximal phone-book where we observe at least 90\% exact accuracy.

\subsection{Memorization: total parameters predict performance}

We train dense transformers and MoEs on phone-books of different sizes and at test time, evaluate whether they memorize the phone number of some names. 
\autoref{fig:phonebook_total_params} reports the maximal phone-book size where the model manages to get an accuracy greater than $90\%$. This gives us an estimate of the memorization capacity of the model.

The findings are clear: no matter the number of active parameters, MoEs match the performance of dense transformers with the same number of total parameters. This suggests that MoEs are able to effectively leverage the extra parameters in additional experts by routing tokens to the experts that contain the necessary information from the training corpus. 
This scaling is remarkable in this case since it even holds when we are only routing to 2 out of 64 experts. For instance, we find that an MoE model with only 42M active parameters outperforms a dense model with 10x as many parameters. 
This type of impressively efficient memorization capacity may be a major reason behind the success of MoE architectures.

\subsection{Reasoning: total parameters do not  predict performance} 

We train dense transformers and MoEs on the shortest path task and then query the models to find the shortest paths in novel, held-out graphs. \autoref{fig:shortest_path_n50_total_params} reports the performance on graphs with 50 nodes with respect to their number of total parameters. Contrary to the phone-book experiment, increasing the number of experts does not consistently improve the performance of MoEs. 
Essentially, we find that active parameters rather than total parameters is a better predictor of performance for these reasoning tasks. 

To connect back to the theory from \autoref{sec:theory}, note that active parameters is directly determined by the width of the network since we always route to exactly 2 experts and fix the depth. Thus, these results corroborate the theory by showing that width (i.e. active parameters) determines the performance on these graph reasoning problems and that increasing the number of experts is not helpful. 
In \autoref{sec:pretrain}, we will further corroborate this idea through evaluation of pre-trained models on commonsense and math reasoning benchmarks.

%% file: pretrain.tex
\section{Pre-trained Models}\label{sec:pretrain}

In this section, we pre-train dense transformers and MoEs and compare their performance on standard math and natural language benchmarks. 
We break the downstream tasks into those that require more memorization and those that require more reasoning. Here, memorization refers to recall in that we measure the ability of the models to retrieve real-world facts,
The memorization-intensive tasks test for ``world knowledge'' and consist of benchmarks like TriviaQA \citep{joshi2017triviaqa}.
We break the reasoning-intensive tasks into two subcategories: one for natural language reasoning tasks like WinoGrande \citep{sakaguchi2021winogrande} and another for mathematical reasoning tasks like Hendrycks-MATH \citep{hendrycks2021measuring}. 
These tasks may be seen as real-world analogs of the stylized phone-book and shortest path tasks studied in \autoref{sec:synthetic}. 

We observe that performance on world-knowledge tasks is governed by the total number of parameters while performance on reasoning tasks depends more on the number of active parameters (\autoref{fig:introduction}). 
Additionally, we conduct an experiment that indicates memorization from MoEs may be harming reasoning performance since there is a larger gap between train and test accuracy for MoEs than dense models at fixed total parameters (\autoref{fig:generalization_gap}).
Finally, we conduct an ablation where we compare models at fixed validation perplexity rather than model size. We find that MoEs perform better on world knowledge tasks and similarly on reasoning tasks compared to dense models (\autoref{fig:iso_ppl_math_nlp}).

\subsection{Setup}

\paragraph{Architecture.} We train dense transformers and MoEs using the OLMoE
codebase \citep{muennighoff2024olmoeopenmixtureofexpertslanguage}. We set the number of layers $L=20$ and vary the width $d\in\{256,512,1024,2048,4096\}$ for dense transformers and $d\in\{256,512,1024\}$ for MoEs. Similarly to \cite{muennighoff2024olmoeopenmixtureofexpertslanguage}, we consistently set the intermediate dimension in the FFN/MoE blocks to $d$ (and not $4d$). For MoEs, we vary the number of experts  $E\in \{8,16,32,64\}$. For the specific case of width 256, we 
 also train a MoE with 256 experts because its parameter count approximately matches the one of a width-2048 dense model and thus, we can compare the downstream performance of the two models. We use top-2 token-choice routing, without token dropping which is implemented in the dMoE function from  the Megablocks package \citep{gale2023megablocks}. We leave the study of MoEs trained with other routing mechanisms for future work.

\paragraph{Training hyperparameters.} We use the AdamW optimizer \citep{loshchilov2017fixing} with a weight decay equal to 0.1. We set the learning rate to $0.001$, train on 63B tokens (60k steps) with batch size 512 and sequence length of 2048. We use warmup during the 20\% first training steps and a linear decay scheduler. We train our models using FSDP \citep{zhao2023pytorch}.

\paragraph{Pre-training datasets.} We train two collections of models, one collection on natural language and another one on math. The natural language dataset is a mixture  constituted of FineWeb-edu \citep{penedo2024fineweb}, Cosmopedia \citep{benallal2024cosmopedia}, Wikipedia and the training sets of the downstream tasks we evaluate on. The math dataset is a mixture made of Proof-Pile 2 \citep{azerbayev2023llemma} and instruction datasets such as OpenMathInstruct \citep{toshniwal2024openmathinstruct} and MetaMathQA \citep{yu2023metamath}. Each of the two training mixture approximately totals 65B tokens. A precise description of the training mixtures can be found in \autoref{app:train_data}.

 \paragraph{Evaluation.} We measure the validation perplexity on 5,000 held-out sequences sampled from the training distribution. And we evaluate our models on a series of natural language and math benchmarks. Explicitly, we divide them into three categories: 
 \begin{itemize}[leftmargin=*, itemsep=1pt, topsep=1pt, parsep=1pt]
      \item[--] World-knowledge tasks: TriviaQA \citep{joshi2017triviaqa}, Natural Questions \citep{kwiatkowski2019natural}, HotpotQA \citep{yang2018hotpotqa}, WebQuestions \citep{berant2013semantic}, ComplexWebQuestions \citep{talmor2018web}.
      \item[--] Commonsense tasks: ARC-C and ARC-E \citep{clark2018think}, CommonsenseQA \citep{talmor2018commonsenseqa}, HellaSwag \citep{zellers2019hellaswag}, OpenbookQA \citep{mihaylov2018can}, PIQA \citep{bisk2020piqa}, SciQ \citep{welbl2017crowdsourcing}, SIQA \citep{sap2019socialiqa}, WinoGrande \citep{sakaguchi2021winogrande}.
      \item[--] Math: SVAMP \citep{patel2021nlp}, GSM8k \citep{cobbe2021training}, GSM-Hard \citep{gao2023pal}, Hendrycks-MATH  \citep{hendrycks2021measuring}, Minerva-MATH \citep{lewkowycz2022solving}.
 \end{itemize}

\begin{wrapfigure}[23]{r}{5.5cm}
\vspace*{-.6cm}
 
\centering
\begin{subfigure}[t]{0.30\textwidth}

\hspace*{-.4cm}\includegraphics[height=3.5cm]{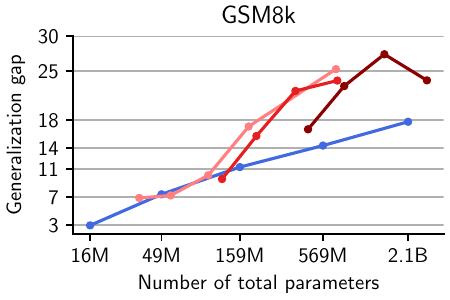}

\vspace{-.1cm}\caption*{(a)}

\end{subfigure}

\vspace{.1cm}

\begin{subfigure}[t]{0.30\textwidth}

\hspace*{-.4cm}\includegraphics[height=3.5cm]{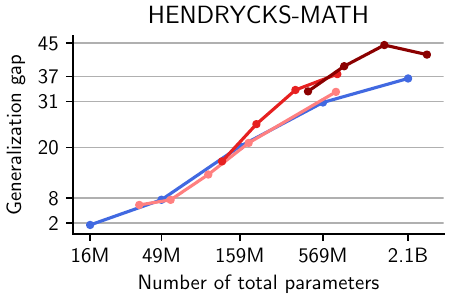}
\caption*{(b)}
\end{subfigure}
 \vspace*{-.2cm}
  \scriptsize  
        \mbox{\hspace*{-.0cm}\cblock{65}{105}{225}\hspace{1mm}Dense transformer}

        \vspace{.2cm}
        
        \mbox{\hspace*{1.1cm}\cblock{255}{128}{ 128}\hspace{1mm}MoE (18M active parameters)}

        \mbox{\hspace*{1.1cm}\cblock{230}{32}{ 32}\hspace{1mm}MoE (58M active parameters)}

        \vspace{.0cm}
        
        \mbox{\hspace*{1.2cm}\cblock{139}{0}{ 0}\hspace{1mm}MoE (200M active parameters)}

\vspace{-.3cm}
        
\caption{\small Generalization gap  when the test set is GSM8k (a) and Hendrycks-MATH (b).} 
\vspace{-0.cm}
\label{fig:generalization_gap}
\end{wrapfigure}

In all our experiments, we plot the average accuracy for each of these three categories. We report the corresponding per-task performance in \autoref{app:experiments}.

\subsection{Results}

\paragraph{Experts improve memorization more than reasoning.} We observe that the conclusions from our theoretical results and synthetic experiments also hold when pre-training and evaluating language models on natural language and math. In \autoref{fig:world_knowledge_intro}, we report the accuracy of our models with respect to the number of \emph{total} parameters. All the lines in the plot approximately coincide which implies that regardless of the number of active parameters, MoEs can effectively use their routing to leverage all of their parameters to solve memory-intensive tasks. On the other hand, on commonsense and math benchmarks (Figures \ref{fig:commonsense_intro},\ref{fig:math_intro}) we find that MoEs do not reach the performance of dense models with the same number of total parameters. This indicates that for these reasoning tasks, increasing the dense model width is more effective that adding experts.

\paragraph{On math tasks, MoEs display a higher train-test gap than dense models, suggestive of memorization.} 
We provide additional evidence that memorization occurs in pre-trained MoEs by considering the generalization gap. 
In \autoref{fig:generalization_gap} we select 6,319 random problems from the OpenMathInstruct dataset, which is part of the training mixture data. More precisely, we pick 5,000 Hendrycks-MATH like examples and 1,319 GSM8k-like examples to ensure that the number of training examples matches with the corresponding number of examples in GSM8k and Hendrycks-MATH test sets. We then report the \emph{generalization gap}, which is the gap between the accuracy on training examples and test examples. While both dense transformers and MoEs make a \emph{single} pass on the OpenMathInstruct dataset, \autoref{fig:generalization_gap} shows that at scales beyond 159M parameters, MoEs suffer from a more significant generalization gap than dense transformers. This is suggestive that MoEs are more prone to overfit to the problems they have been pre-trained on than dense models.
 
\paragraph{MoE models excel at world knowledge tasks but match dense models in reasoning when perplexity is fixed.}
Finally, we focus on the relationship between validation perplexity and downstream performance in \autoref{fig:iso_ppl_math_nlp}. 
Rather than comparing models by their parameter count, we can compare them based on how well they fit the training distribution as measured by validation perplexity.
Even though two models may have the same perplexity, they will have learned different functions. The question is then if we can see any high level patterns in which types of functions a particular model class is more likely to learn. \autoref{fig:iso_ppl_world_knowledge} shows that at a fixed perplexity, the MoE models outperform the dense models on world knowledge tasks.

On the other hand, Figures \ref{fig:iso_ppl_commonsense} and \ref{fig:iso_ppl_math} show that MoEs and dense models perform about the same on the reasoning tasks at fixed validation perplexity. Note that both dense Transformers and MoEs are trained with the objective of minimizing the perplexity (or loss). However, there could be multiple strategies to achieve the same loss, e.g. by memorizing pieces of data or by improving ``reasoning'' capabilities. Which strategy the model prefers is determined by the implicit bias of the architecture. In the experiments, we see that for memorization / factual recall, MoEs achieve better accuracy for the same pereplexity value, suggesting that they ``prioritize'' memorization over other capabilties.


\begin{figure*}[t]
\centering
\begin{subfigure}[t]{0.31\textwidth}
\centering
\vskip 0pt
    \includegraphics[height=3.cm]{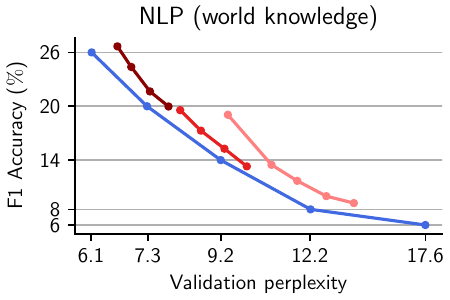}
   \scriptsize
    \vspace*{-.2cm}
    
        \mbox{\hspace*{1.cm}\cblock{65}{105}{225}\hspace{1mm}Dense transformer \hspace{1.5mm} \cblock{255}{128}{ 128}\hspace{1mm}MoE (18M active parameters) \hspace{1.5mm} \cblock{230}{32}{ 32}\hspace{1mm}MoE (58M active parameters)\hspace{1.5mm}  \cblock{139}{0}{ 0}\hspace{1mm}MoE (200M active parameters)}
    \vspace{-.1cm}
    \caption{}
    \label{fig:iso_ppl_world_knowledge}
\end{subfigure}
\hfill
\begin{subfigure}[t]{0.31\textwidth}
\centering
\vskip 0pt
\includegraphics[height=3.cm]{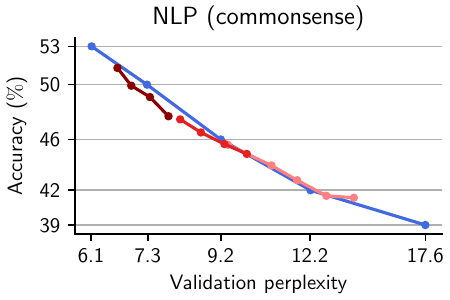}

\vspace*{.1cm}

\hspace*{1.5cm}\caption{}\label{fig:iso_ppl_commonsense}
\end{subfigure}
\hfill
\begin{subfigure}[t]
{0.31\textwidth}
\centering
\vskip 0pt
\includegraphics[height=3.cm]{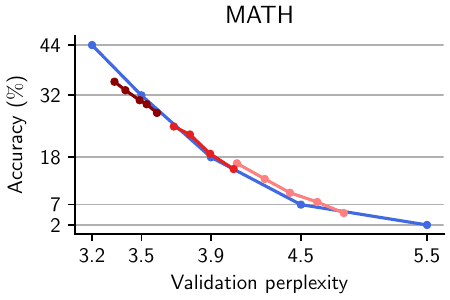}
\vspace{.1cm}
\caption{}
\label{fig:iso_ppl_math}
\end{subfigure}
\caption{(a) On world knowledge benchmarks, MoEs consistently outperform dense transformers in downstream performance when fixing the validation perplexity. (b-c) In reasoning benchmarks, dense transformers perform about the same as MoEs at a fixed validation perplexity. MoEs can achieve these perplexities with less active parameters, but may require substantially more total parameters.}\label{fig:iso_ppl_math_nlp}
\end{figure*}

%% file: discussion.tex
\section{Discussion}

In recent years, scaling up the number of parameters in Transformers has been the dominant approach for improving performance on language modeling. A standard Transformer of dimension $d$ and sequence length $L$ has number of parameters which scales with $O(d^2)$, and run-time that scales with $O(d^2 L^2)$. 
Improving the efficiency can either attempt to reduce the dependence on $ L$ or $ d$. 
Sub-quadratic attention variants attempt to improve dependence on $ L $ \citep{katharopoulos2020transformers, peng2023rwkv, fu2022hungry,gu2023mamba}, while MoEs attempt to improve dependence on $ d$ by scaling the number of parameters without scaling the dimension of the model.

This paper illuminates the costs and benefits of this reduced dependence on $ d$. 
We show that for some reasoning-intensive tasks increasing the dimension $d$ is inevitable, and scaling the computation with $O(d^2)$ seems unavoidable. 
This remains true regardless of the different design choices in the MoE architecture and is backed up empirically. 
There is increasing interest in developing non-MoE models with sub-quadratic dependence on $d$, using some structural assumptions on the weight layers \citep{kamalakara2022exploring, dao2021pixelated, dao2022monarch, fu2024monarch}, which could provide an alternative.

On the other hand, we find that MoEs are highly effective at knowledge intensive tasks. They are able to much more efficiently memorize facts than dense models with a similar number of active parameters, even matching the performance of dense models with the same number of total parameters. This suggests that MoEs are valuable as memorization machines and perhaps this particular capability can be leveraged while relying on other architectures for more reasoning-intensive tasks.

%% file: acknowledgements.tex
\section*{Acknowledgements}

We thank Cyril Zhang for helpful discussions and Max Shad for his support while running the experiments. Kempner Institute computing
resources enabled this work. Samy Jelassi acknowledges funding supported by the Center of Mathematical Sciences and Applications. Alex Gu is supported by the National Science Foundation (NSF) Graduate Research Fellowship under
Grant No. 2141064. David Alvarez-Melis acknowledges support from the Kempner Institute, the Aramont Fellowship Fund, and the FAS Dean's Competitive Fund for Promising Scholarship.  This work has been made possible in part by a gift from the Chan Zuckerberg Initiative Foundation to establish the Kempner Institute
for the Study of Natural and Artificial Intelligence; support from the Office of Naval Research under award N00014-22-1-2377, and the
National Science Foundation Grant under award \#IIS 2229881.

%% file: app_related_work.tex
\section{Limitations and future work}

 While we provide substantial experiments on a wide range of tasks, many of these tasks are not pure memorization nor pure reasoning. We study these two extreme cases to better convey our message but it would be interesting in future work to understand the difference between dense and sparse models on tasks that mix both memorization and recall. Our pre-trained models have up to $\leq2.1$B parameters, but we recognize that large scale MoEs like Mixtral \citep{jiang2024mixtral}, DeepSeek-V2 \citep{dai2024deepseekmoe}, and others have orders of magnitude more parameters. We hypothesize that our results would still be meaningful at larger scales due to the strong theoretical underpinning, but it is not guaranteed. Moreover, as suggested above, it would be an interesting direction for future work to propose new architectures with reduced $ d $ dependence that can get the best of both worlds and solve reasoning and memorization tasks.

%% file: app_experiments.tex
\section{Details on the pre-training datasets}\label{app:train_data}

In \autoref{sec:pretrain}, we pretrain two collections of models, one on ``natural language" and the other on ``math". Here, we give a precise breakdown of our training mixtures. We start with the ``natural language" training mixture that totals 64B tokens:
\begin{itemize}
   \item[--] 37B tokens from Fineweb-edu dedup \citep{penedo2024fineweb}.
   \item[--] 14B tokens from Cosmopedia \citep{benallal2024cosmopedia}.
   \item[--] 12B tokens from Wikipedia (we loop over Wikipedia 3 times).
   \item[--] 1B tokens from the training set of the downstream tasks we test on. We create 3 copies of each of these to increase their presence in the mixture. The presence of these datasets is pretty important as argued in \cite{allen2023physics} so that the model is familiar with the downstream tasks at test time.
   \begin{itemize}
       \item[$\ast$] ComplexWebQuestions training set \citep{talmor2018web} 
       \item[$\ast$] HotPotQA training set \citep{yang2018hotpotqa}
       \item[$\ast$] Natural Questions training set \citep{kwiatkowski2019natural}
       \item[$\ast$] TriviaQA training set \citep{joshi2017triviaqa}
       \item[$\ast$] WebQuestions training set \citep{berant2013semantic}
       \item[$\ast$] ARC-Easy and ARC-Challenge training sets \citep{clark2018think}
       \item[$\ast$] Hellaswag training set \citep{zellers2019hellaswag}
       \item[$\ast$] OpenBookQA training set \citep{mihaylov2018can}
       \item[$\ast$] PIQA training set \citep{bisk2020piqa}
       \item[$\ast$] SciQ training set \citep{welbl2017crowdsourcing}
       \item[$\ast$] SIQA training set \citep{sap2019socialiqa}
       \item[$\ast$] Winogrande training set \citep{sakaguchi2021winogrande}
   \end{itemize}
\end{itemize}

Our ``math" training mixture that totals 66B tokens gathers:
\begin{itemize}
       \item[--] 55B tokens from Proof-Pile 2 \citep{azerbayev2023llemma} that contain AlgebraicStack (11B), OpenWebMath \citep{paster2023openwebmath} and ArXiv (29B).
       \item[--] 2B tokens from OpenMathInstruct-1: we select the instances with a  correct answer from the training set  \citep{toshniwal2024openmathinstruct}
       \item[--] 7B tokens from DeepMind math \citep{saxton2019analysing}
       \item[--] 2B tokens from the following instruction-like datasets:
          \begin{itemize}
              \item[$\ast$] Math-Orca \citep{mitra2024orcamath}
              \item[$\ast$] TinyGSM \citep{liu2023tinygsm} (we only select 1 million examples from there).
              \item[$\ast$] StackMathQA \citep{stackmathqa2024}
              \item[$\ast$] MAmmoTH2 \citep{yue2024mammoth2}  (we only select the mathstackexchange subset).
              \item[$\ast$] NuminaMath-CoT \citep{numina2024winning} (duplicated 3 times)
              \item[$\ast$] MetaMathQA \citep{yu2023metamath} (duplicated 3 times)
          \end{itemize}
\end{itemize}


\section{Additional experiments}\label{app:experiments}

In all our experiments in \autoref{sec:pretrain}, we report the average accuracy performance obtained by our pre-trained models on respectively world knowledge, commonsense and math benchmarks. Here, we provide the results per task. In \autoref{app:total_params_per_task}, we display for each task, the downstream performance on a per parameter basis (similar to \autoref{fig:introduction}) and in \autoref{app:iso_ppl_per_task}, we plot for each task, the downstream performance on a per validation perplexity basis (similar to \autoref{fig:iso_ppl_math_nlp}).

\subsection{Downstream performance on a per parameter basis}\label{app:total_params_per_task}

\begin{figure*}[!h]
\centering
\begin{subfigure}[!h]{0.31\textwidth}
\centering
\vskip 0pt
    \includegraphics[height=3.cm]{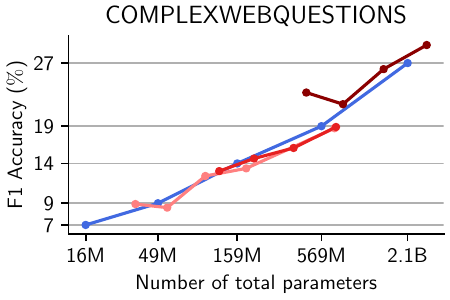}
\end{subfigure}
\hfill
\begin{subfigure}[!h]
{0.31\textwidth}
\centering
\vskip 0pt
\includegraphics[height=3.cm]{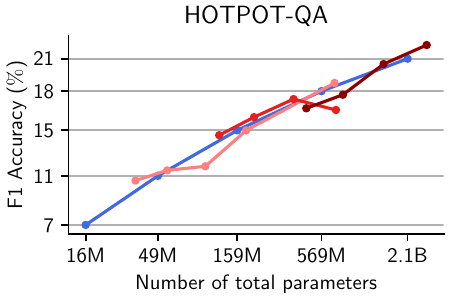}
\end{subfigure}
\begin{subfigure}[!h]
{0.31\textwidth}
\centering
\vskip 0pt
\includegraphics[height=3.cm]{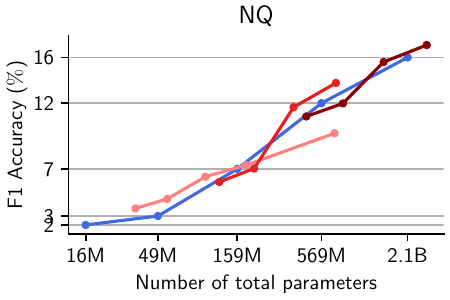}
\end{subfigure}
\begin{subfigure}[!h]
{0.31\textwidth}
\centering
\vskip 0pt
\includegraphics[height=3.cm]{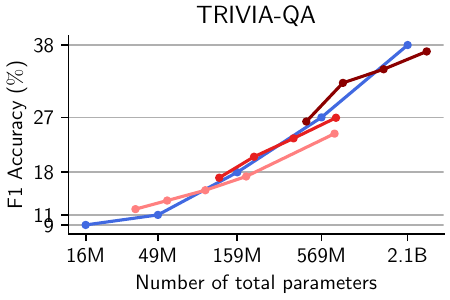}
\end{subfigure}
\begin{subfigure}[!h]
{0.31\textwidth}
\centering
\vskip 0pt
\includegraphics[height=3.cm]{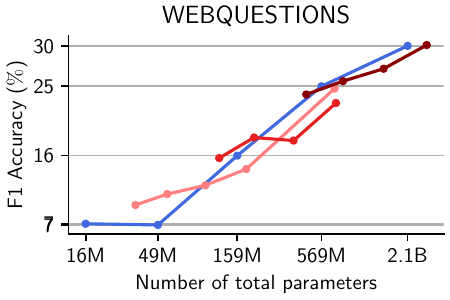}
\end{subfigure}
 \scriptsize
    \vspace*{.2cm}
    
        \mbox{\hspace*{1.cm}\cblock{65}{105}{225}\hspace{1mm}Dense transformer \hspace{1.5mm} \cblock{255}{128}{ 128}\hspace{1mm}MoE (18M active parameters) \hspace{1.5mm} \cblock{230}{32}{ 32}\hspace{1mm}MoE (58M active parameters)\hspace{1.5mm}  \cblock{139}{0}{ 0}\hspace{1mm}MoE (200M active parameters)}
\caption{Downstream performance on the world knowledge tasks with respect to the total number of parameters of the models.}\label{fig:total_params_world_knowledge_per_task}
\end{figure*}

\begin{figure*}[!h]
\vspace{-0.2cm}

\centering
\begin{subfigure}[!h]
{0.31\textwidth}
\centering
\vskip 0pt
\includegraphics[height=3.cm]{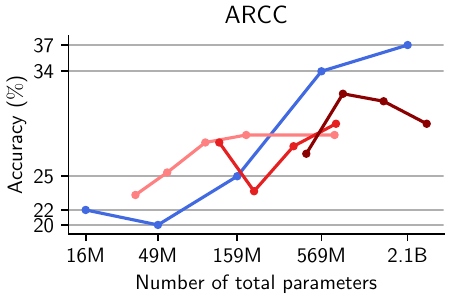}
\end{subfigure}
\begin{subfigure}[!h]
{0.31\textwidth}
\centering
\vskip 0pt
\includegraphics[height=3.cm]{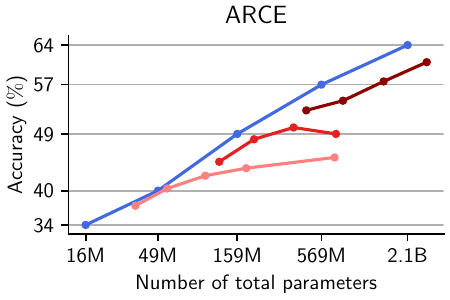}
\end{subfigure}
\begin{subfigure}[!h]
{0.31\textwidth}
\centering
\vskip 0pt
\includegraphics[height=3.cm]{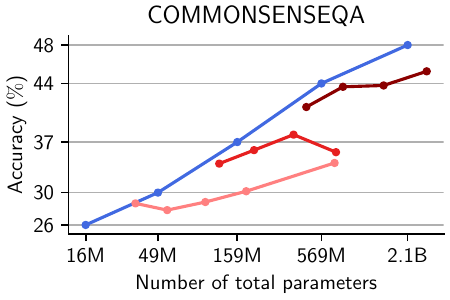}
\end{subfigure}
\begin{subfigure}[!h]
{0.31\textwidth}
\centering
\vskip 0pt
\includegraphics[height=3.cm]{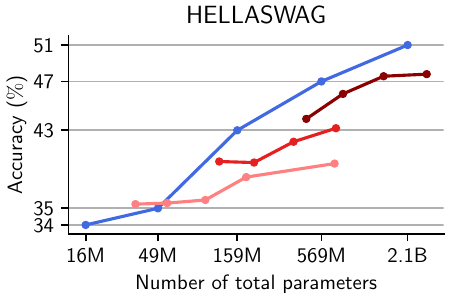}
\end{subfigure}
\begin{subfigure}[!h]
{0.31\textwidth}
\centering
\vskip 0pt
\includegraphics[height=3.cm]{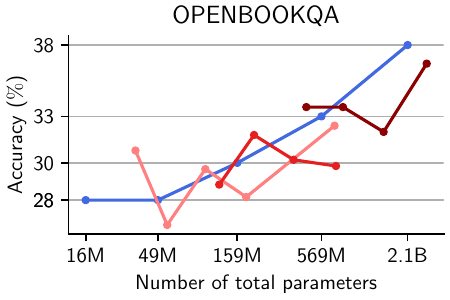}
\end{subfigure}
\begin{subfigure}[!h]
{0.31\textwidth}
\centering
\vskip 0pt
\includegraphics[height=3.cm]{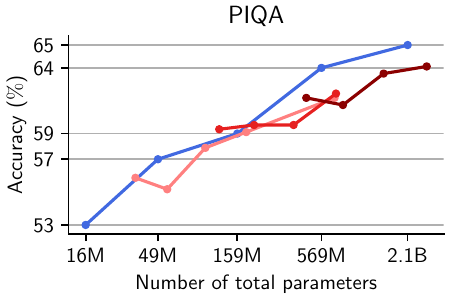}
\end{subfigure}
\begin{subfigure}[!h]
{0.31\textwidth}
\centering
\vskip 0pt
\includegraphics[height=3.cm]{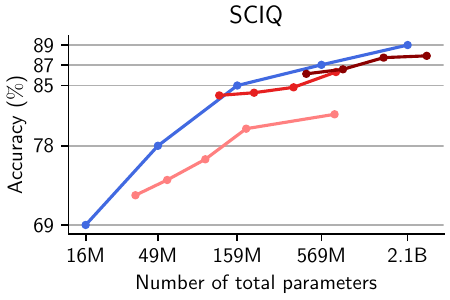}
\end{subfigure}
\begin{subfigure}[!h]
{0.31\textwidth}
\centering
\vskip 0pt
\includegraphics[height=3.cm]{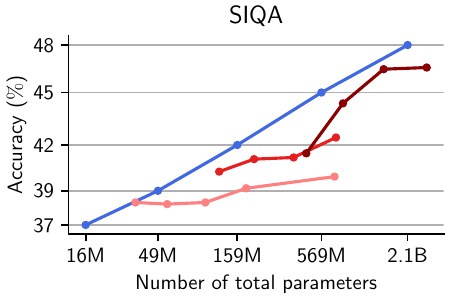}
\end{subfigure}
\begin{subfigure}[!h]
{0.31\textwidth}
\centering
\vskip 0pt
\includegraphics[height=3.cm]{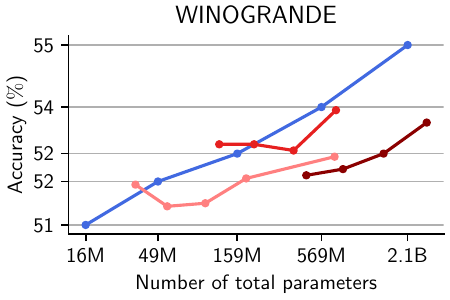}
\end{subfigure}
 \scriptsize
    \vspace*{0.1cm}
    
        \mbox{\hspace*{1.cm}\cblock{65}{105}{225}\hspace{1mm}Dense transformer \hspace{1.5mm} \cblock{255}{128}{ 128}\hspace{1mm}MoE (18M active parameters) \hspace{1.5mm} \cblock{230}{32}{ 32}\hspace{1mm}MoE (58M active parameters)\hspace{1.5mm}  \cblock{139}{0}{ 0}\hspace{1mm}MoE (200M active parameters)}

\caption{Downstream performance on the commonsense tasks with respect to the total number of parameters of the models.}\label{fig:total_params_nlp_reasoning_per_task}
\end{figure*}

\begin{figure*}[!h]
\centering
\begin{subfigure}[!h]
{0.31\textwidth}
\centering
\vskip 0pt
\includegraphics[height=3.cm]{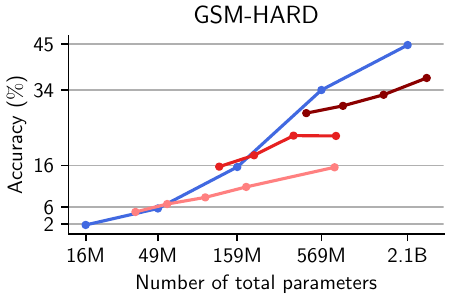}
\end{subfigure}
\begin{subfigure}[!h]
{0.31\textwidth}
\centering
\vskip 0pt
\includegraphics[height=3.cm]{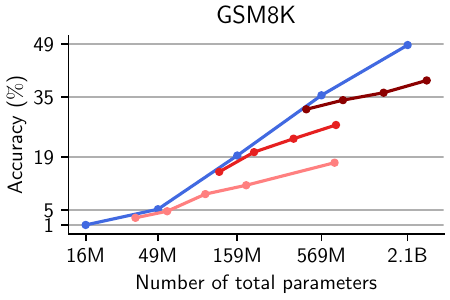}
\end{subfigure}
\begin{subfigure}[!h]
{0.31\textwidth}
\centering
\vskip 0pt
\includegraphics[height=3.cm]{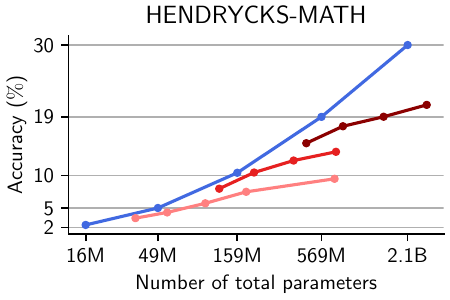}
\end{subfigure}
\begin{subfigure}[!h]
{0.31\textwidth}
\centering
\vskip 0pt
\includegraphics[height=3.cm]{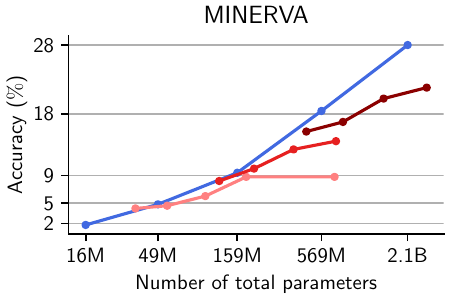}
\end{subfigure}
\begin{subfigure}[!h]
{0.31\textwidth}
\centering
\vskip 0pt
\includegraphics[height=3.cm]{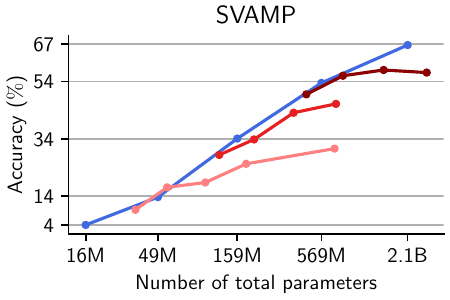}
\end{subfigure}

\scriptsize
    \vspace*{0.1cm}
    
        \mbox{\hspace*{1.cm}\cblock{65}{105}{225}\hspace{1mm}Dense transformer \hspace{1.5mm} \cblock{255}{128}{ 128}\hspace{1mm}MoE (18M active parameters) \hspace{1.5mm} \cblock{230}{32}{ 32}\hspace{1mm}MoE (58M active parameters)\hspace{1.5mm}  \cblock{139}{0}{ 0}\hspace{1mm}MoE (200M active parameters)}

\caption{Downstream performance on the math benchmarks with respect to the total number of parameters of the models.}\label{fig:total_params_math_per_task}
\end{figure*}

\newpage

\subsection{Downstream performance on a per val perplexity basis}\label{app:iso_ppl_per_task}

\begin{figure*}[!h]
\centering
\begin{subfigure}[!h]{0.31\textwidth}
\centering
\vskip 0pt
    \includegraphics[height=3.cm]{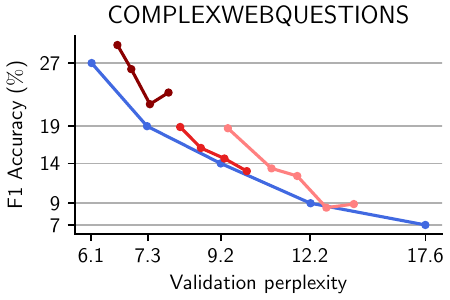}
\end{subfigure}
\hfill
\begin{subfigure}[!h]
{0.31\textwidth}
\centering
\vskip 0pt
\includegraphics[height=3.cm]{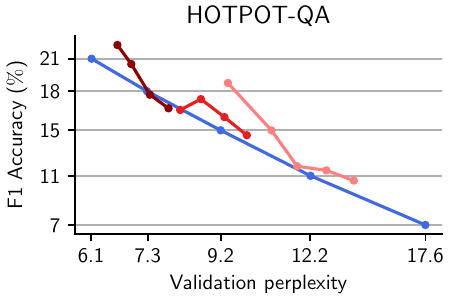}
\end{subfigure}
\begin{subfigure}[!h]
{0.31\textwidth}
\centering
\vskip 0pt
\includegraphics[height=3.cm]{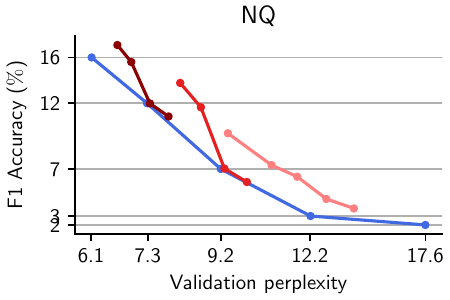}
\end{subfigure}
\begin{subfigure}[!h]
{0.31\textwidth}
\centering
\vskip 0pt
\includegraphics[height=3.cm]{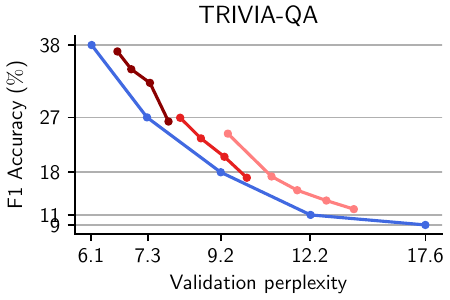}
\end{subfigure}
\begin{subfigure}[!h]
{0.31\textwidth}
\centering
\vskip 0pt
\includegraphics[height=3.cm]{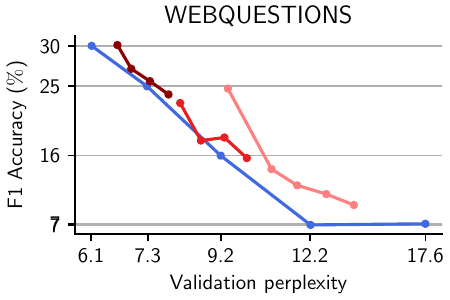}
\end{subfigure}

\scriptsize
    \vspace*{0.1cm}
    
        \mbox{\hspace*{1.cm}\cblock{65}{105}{225}\hspace{1mm}Dense transformer \hspace{1.5mm} \cblock{255}{128}{ 128}\hspace{1mm}MoE (18M active parameters) \hspace{1.5mm} \cblock{230}{32}{ 32}\hspace{1mm}MoE (58M active parameters)\hspace{1.5mm}  \cblock{139}{0}{ 0}\hspace{1mm}MoE (200M active parameters)}

\caption{Downstream performance on the world knowledge tasks with respect to the validation perplexity.}\label{fig:iso_ppl_world_knowledge_per_task}
\end{figure*}

\begin{figure*}[!h]
\centering
\begin{subfigure}[!h]
{0.31\textwidth}
\centering
\vskip 0pt
\includegraphics[height=3.cm]{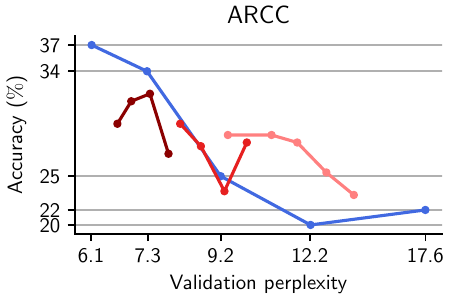}
\end{subfigure}
\begin{subfigure}[!h]
{0.31\textwidth}
\centering
\vskip 0pt
\includegraphics[height=3.cm]{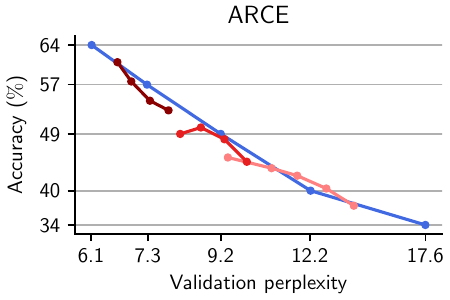}
\end{subfigure}
\begin{subfigure}[!h]
{0.31\textwidth}
\centering
\vskip 0pt
\includegraphics[height=3.cm]{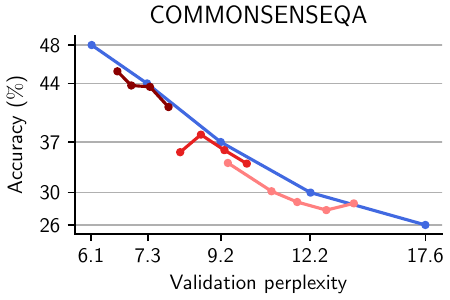}
\end{subfigure}
\begin{subfigure}[!h]
{0.31\textwidth}
\centering
\vskip 0pt
\includegraphics[height=3.cm]{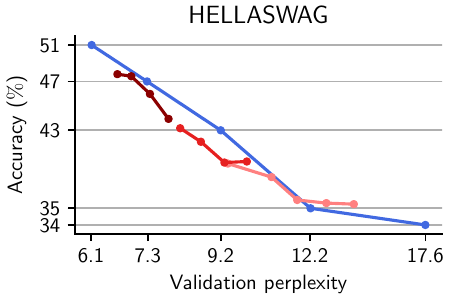}
\end{subfigure}
\begin{subfigure}[!h]
{0.31\textwidth}
\centering
\vskip 0pt
\includegraphics[height=3.cm]{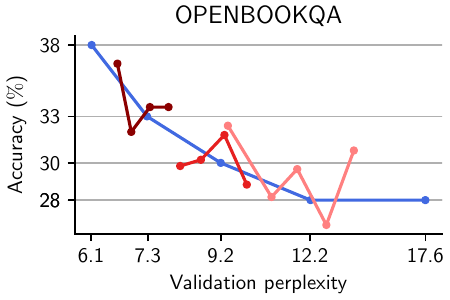}
\end{subfigure}
\begin{subfigure}[!h]
{0.31\textwidth}
\centering
\vskip 0pt
\includegraphics[height=3.cm]{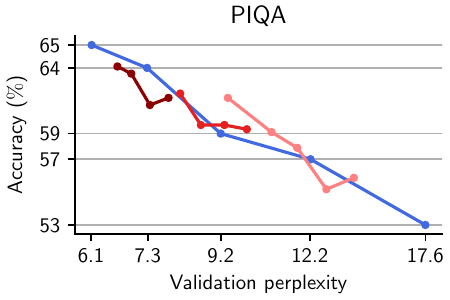}
\end{subfigure}
\begin{subfigure}[!h]
{0.31\textwidth}
\centering
\vskip 0pt
\includegraphics[height=3.cm]{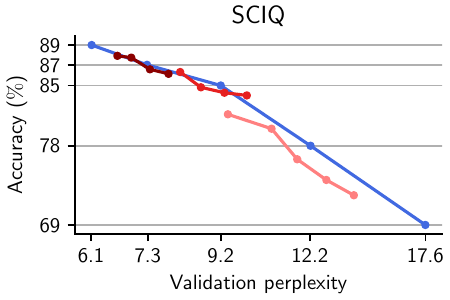}
\end{subfigure}
\begin{subfigure}[!h]
{0.31\textwidth}
\centering
\vskip 0pt
\includegraphics[height=3.cm]{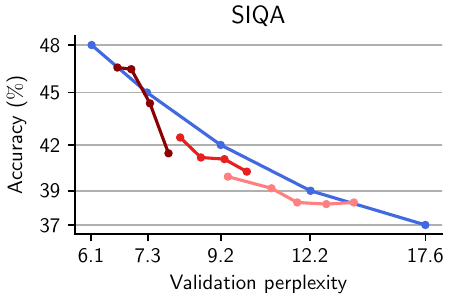}
\end{subfigure}
\begin{subfigure}[!h]
{0.31\textwidth}
\centering
\vskip 0pt
\includegraphics[height=3.cm]{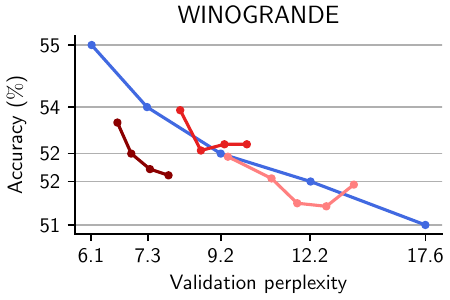}
\end{subfigure}

\scriptsize
    \vspace*{0.1cm}
    
        \mbox{\hspace*{1.cm}\cblock{65}{105}{225}\hspace{1mm}Dense transformer \hspace{1.5mm} \cblock{255}{128}{ 128}\hspace{1mm}MoE (18M active parameters) \hspace{1.5mm} \cblock{230}{32}{ 32}\hspace{1mm}MoE (58M active parameters)\hspace{1.5mm}  \cblock{139}{0}{ 0}\hspace{1mm}MoE (200M active parameters)}

\caption{Performance on the commonsense tasks with respect to the validation perplexity.}\label{fig:iso_ppl_nlp_reasoning_per_task}
\end{figure*}

\begin{figure*}[!h]
\centering
\begin{subfigure}[!h]
{0.31\textwidth}
\centering
\vskip 0pt
\includegraphics[height=3.cm]{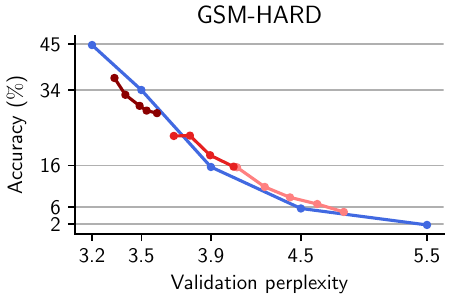}
\end{subfigure}
\begin{subfigure}[!h]
{0.31\textwidth}
\centering
\vskip 0pt
\includegraphics[height=3.cm]{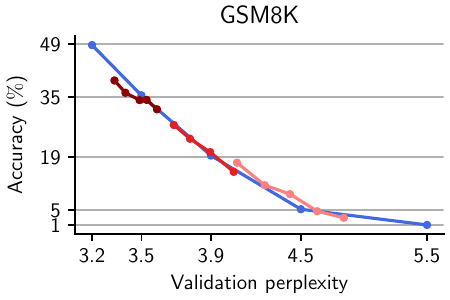}
\end{subfigure}
\begin{subfigure}[!h]
{0.31\textwidth}
\centering
\vskip 0pt
\includegraphics[height=3.cm]{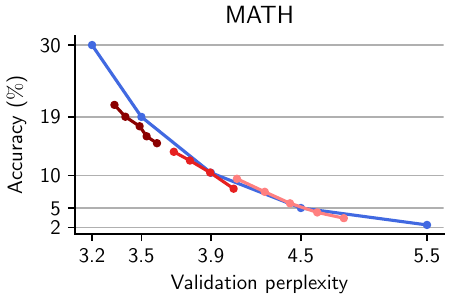}
\end{subfigure}
\begin{subfigure}[!h]
{0.31\textwidth}
\centering
\vskip 0pt
\includegraphics[height=3.cm]{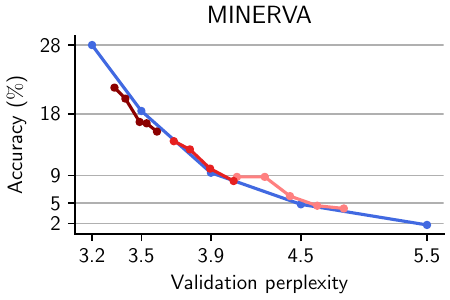}
\end{subfigure}
\begin{subfigure}[!h]
{0.31\textwidth}
\centering
\vskip 0pt
\includegraphics[height=3.cm]{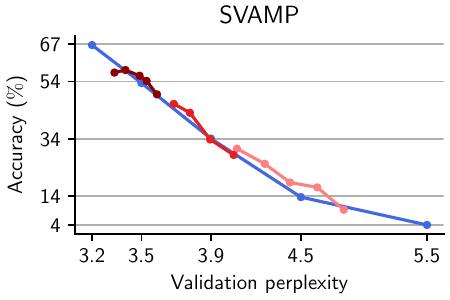}
\end{subfigure}

\scriptsize
    \vspace*{0.1cm}
    
        \mbox{\hspace*{1.cm}\cblock{65}{105}{225}\hspace{1mm}Dense transformer \hspace{1.5mm} \cblock{255}{128}{ 128}\hspace{1mm}MoE (18M active parameters) \hspace{1.5mm} \cblock{230}{32}{ 32}\hspace{1mm}MoE (58M active parameters)\hspace{1.5mm}  \cblock{139}{0}{ 0}\hspace{1mm}MoE (200M active parameters)}

\caption{Downstream performance on the math benchmarks with respect to the validation perplexity.}\label{fig:iso_ppl_math_per_task}
\end{figure*}

\newpage\phantom{blabla}
\newpage\phantom{blabla}

%% file: app_proofs.tex
\section{Proofs}\label{app:proofs}
\subsection{Reasoning proofs}

\begin{definition}[Set-disjointness task]
     Set disjointness is the following task: given two inputs $A, B \in \{0, 1\}^r$ for some $r \in \Nset$, compute $\max_i A_i B_i$. 
\end{definition}
Set-disjointness can be thought of as follows: Alice and Bob are given sets $A$ and $B$ respectively. Their objective is to determine whether they have any overlapping items in their sets. 
\begin{lemma}[Equivalence of set-disjointness and length-2 path]
    \label{lemma:SD_L2path}
    The set-disjointness task is equivalent to the length-2 path task.
\end{lemma}

\begin{proof}
    $(\implies)$: Given an instance of set-disjointness, we can encode it into a length-2 path problem. Denote every item $i$ as a vertex. Denote two extra vertices as $A$, $B$, corresponding to Alice and Bob. For every element $i$ that Alice has, draw an edge between $A$ and $i$. For every element $i$ that Bob has, draw an edge between $B$ to $i$. If and only if there are any overlapping elements, then there is a length-2 path from $A$ to $B$. The number of elements because the number of vertices that do not belong to Alice or Bob.
    
    $(\impliedby)$: Consider an instance $G = (V, E), s, d$ of length-2 path, where $s$ is the source vertex and $d$ is the sink vertex. For all vertices with an edge with $s$, put this element into Alice's set of elements. For all vertices with an edge with $d$, put this element into Bobs's set of elements. If and only if there is a length-2 path, then Alice and Bob's sets are overlapping. Then, $r$ is the number of vertices. 
\end{proof}
\begin{lemma}[Communication complexity lower-bound on concatenated outputs]
\label{lemma:concatenation}
    For some sequence length, fix two disjoint subsets $A, B \subset [N-1]$, and consider a single-layer transformer $f \in \text{Transformer}_{m, H, 1}^N$   with $O(\log N)$-bit precision that solves set disjointness for any input $X$ where $X_A$ is a function of Alice's input $a \in \{0, 1\}^r$, $X_B$ is a function of Bob's input $b \in \{0, 1\}^r$, and $X_{[N] \setminus (A \cup B)}$ is fixed regardless of $a, b$. Then, $f$ has width satisfying $mH = \Omega(r / \log N)$.
\end{lemma}
\begin{proof}
    By re-writing the following, the remainder of the proof from \cite{sanford2024understanding} still holds.
    \[
    \text{DISJ}(a, b) = \psi \paren*{\bracket*{\softmax \paren*{\phi(x_N)^\top Q_h K_h^\top \phi(X)}\phi(X) v_h }_{h \in [H]} }.
    \]
This is because we may still use the same definition for $Z_{h, S}, L_{h, S}$ as in the proof. 
Hence, this concludes the proof.
\end{proof}

\subsubsection{Proof of Theorem \ref{thm:cc_lowerbound}}
We restate the corollary.
\begin{theorem}[Theorem \ref{thm:cc_lowerbound}]
    For some input sequence $G = (V, E)$, fix two disjoint subsets $A, B \subset [N-1]$, and consider a single-layer transformer $f \in \text{Transformer}_{m, H, 1, K}^N$   with $O(\log N)$-bit precision that solves length-2 path for any input $X$ where $X_A$ is a function of edges with the source $s$, $X_B$ is a function of edges with the destination $d$. Then, $f$ has width satisfying $mH = \Omega(\abs{V} / \log N)$.
\end{theorem}
\begin{proof}
The proof outline is as follows:
\begin{enumerate}
    \item Adapt Lemma 39 \citep{sanford2024understanding} to support concatenation instead of addition from different attention heads.
    \item The lower bound with concatenation holds for length-2 path because set-disjointness and length-2 path are equivalent.
    \item Extend the result to sparse transformers.
\end{enumerate}
We complete the first step with Lemma \ref{lemma:concatenation}. We complete the second set due to Lemma \ref{lemma:SD_L2path}. It remains to show that a router function also yields the same lower bound.
We show that Lemma 39 of \cite{sanford2024understanding} can be generalized to the case in which $\psi$ is applied according to a routing function. Specifically, consider a top-$1$ routing function $r: \Rset^m \rightarrow [K]$, and $K$ element-wise functions $\psi_1, \ldots, \psi_K: \Rset^m \rightarrow \Rset$. For shorthand, define:
\[
Y(X_N) = \bracket*{\softmax \paren*{\phi(x_N)^\top Q_h K_h^\top \phi(X)}\phi(X) v_h }_{h \in [H]},
\]
which is the output of the attention head prior to applying the element-wise transformation.
Next, we define $f(X_N)$ as the output when the router function $r$ is used to select $\psi_i$.
\[
f(X_N) =\sum_{i \in K} \I \curl*{r(Y(X_N)) = i} \psi_i(Y(X_N)).
\]
Because the lower bound does not place any restrictions on the function $\psi$ and rather argues a communication-complexity lower bound due to information from $Y(X_N)$, the lower bound also holds for a routing function.
\end{proof}

\subsubsection{Proof of Theorem \ref{thm:mem_ub}}
We re-state Theorem \ref{thm:mem_ub} and give its proof.
\begin{theorem}[Theorem \ref{thm:mem_ub}]
For sequence length $N$, there exists $f \in \text{Transformer}_{m, 1, 1}^N$  with $O(\log N)$-bit precision  and width $\abs{V}$ that solves length-2 path for any input $X$.
\end{theorem}
\begin{proof}

Tokens are elements in $\mathcal{V} = V \cup \{0\} \times V \cup \{0\}$. The input is as follows: for vertex $i$, if the source shares an edge with that vertex, then the $i$'th input value is $(s, i)$. Otherwise, it is $(s, 0)$. The first $\abs{V}$ tokens we see correspond to edges possibly shared with the source vertex. Then, the last $\abs{V}$ input tokens correspond to edges possibly shared with the destination vertex and share the same format as the first $r$ tokens. In between, we can have arbitrary edges $(u, v)$. 
We define an embedding function where $\mathbf{e}_i$ is the $i$'th standard basis vector in dimension $r$.
\begin{align*}
\phi: \ & \mathcal{V} \rightarrow \Rset^{\abs{V}  }\\
& (u, v) \mapsto \begin{cases}
    \mathbf{e}_i &\text{if } i > 0 \text{ and } u = s \text{ or } u = v \\
    \mathbf{0} &\text{if } i = 0.
\end{cases}
\end{align*}
Next, we define $V_h \in \Rset^{\abs{V} \times \abs{V}}$ to be the identity matrix, and $Q_h, V_h \in \Rset^{\abs{V} \times \abs{V}}$ both to have $0$ everywhere. Consequently, the attention matrix is given by:
\[
\paren*{\begin{bmatrix}
    1 / \abs{V} & \ldots & 1/\abs{V} \\
    \vdots & \ddots \\
    1 / \abs{V} &  & 1/\abs{V} 
\end{bmatrix}  \phi(X)}_{j, i} =\begin{cases}
    2 / \abs{V} & \text{if there is a path through $i$} \\
    1/\abs{V} & \text{if one target vertex shares an edge with } i \\
    0 & \text{otherwise.}
\end{cases}
\]
For any entry that exceeds $\frac{1}{\abs{V}}$, the correct answer is there is a length-2 path. Hence, any thresholding function which achieves this separation suffices.  Provided that any rounding $\frac{1}{\abs{V}}$ to $\tilde{c}$ so that it is represented with $O(\log N)$ is sufficient to separate between $\tilde{c}$ and $2\tilde c$, then $O(\log N)$ bits are sufficient.
\end{proof}

\subsection{Memorization Proofs}
\label{appc:memorization}

Assume that $K$ is a power of 2, and let $m = m_0 + \log(K)$. We associate each expert $i$ with a vector $\bv_i \in \{\pm 1\}^{\log(K)}$ and choose $\br_i = (v_1, \dots, v_{\log(K)}, 0, \dots, 0) \in \mathbb{R}^m$.

\begin{lemma}
\label{lemma:prob1k}
    For some $c > 0$, and some input $x \sim \cN(0, c I_m)$, the probability of routing to expert $i$ is $1/K$ for all $i$.
\end{lemma}
\begin{proof}
    By construction, we choose the expert as follows:
    \[
    i = \argmax_{j \in [K]} v_j^\top x.
    \]
    The solution above must be the one whose signs match those of $x$. Because the entries of $x$ are $0$-mean, this implies that the probability of routing to any particular expert is $1/K$. 
\end{proof}

\begin{lemma}
     Fix $\delta \in (0,1)$ and some expert $j$. With probability at least $1-\delta$, the number of examples routed to $j$ is at most $2N/K$, given $N \ge \frac{K^2\ln\paren*{1/\delta} }{2}$ examples from $\mathcal{N}(0, cI_m)$.
\end{lemma}

\begin{proof}
    The proof relies on a Hoeffding bound: for bounded independent random variables $X_1, \ldots, X_N$ such that $a \leq X_i \leq b \ \forall i$, and $S_N = \sum_{i = 1}^N X_i$, we have:
    \[
    \Pr[S_N - \E[S_N] \geq \E[S_N]] \leq \exp \curl*{- \frac{2\E[S_N]^2}{N(b-a)^2}}.
    \]
    In our case, let $X_i = \1\{\text{sample } i \text{ is routed to expert } j\}$, for some fixed $j \in [K]$. Note that $\Pr[X_i = 1] = 1/K$ due to Lemma \ref{lemma:prob1k} and as such $E[S_N] = N/K$. Further, $X_i \in [0, 1]$. Hence we obtain:
    \begin{align*}
    \Pr\bracket*{S_N  \geq \frac{2N}{K}} &\leq \exp \curl*{- \frac{2(N/K)^2}{N}} \\ 
    &\leq \exp \curl*{- \frac{2N}{K^2}}.
    \end{align*}
In order to upper-bound this probability with $\delta$, we obtain:
\begin{align*}
    & \delta \geq \exp \curl*{- \frac{2N}{K^2}} \\
    &\implies \ln\paren*{\frac{1}{\delta}} \leq \frac{2N}{K^2} \\
    &\implies \frac{K^2\ln\paren*{1/\delta} }{2} \leq N.
\end{align*}

\end{proof}

\begin{lemma}
    Fix $\delta \in (0,1)$,  $K \in \Nset$ and the routing function as described in Lemma \ref{lemma:prob1k}. Given $n \geq \frac{K^2\ln\paren*{1/\delta} }{2}$ samples with embedding dimension $m$. For every expert $i$, with probability at least $1-\delta$, there exists an MLP of width $\tilde{O}(n/mK)$ that correctly classifies the examples routed to this expert.
\end{lemma}

\begin{proof}

    We begin with a result from \cite{daniely2020memorizing}. Consider $h$, a depth-2 network with $q$ neurons, an activation function $\phi: \Rset \rightarrow \Rset$, which is $O(1)$-Lipschitz, piecewise twice-differentiable, and satisfies $\E_{X \sim \cN(0, 1)}[\phi'(X)] = 0$, weights $\a_i \in \{\pm 1\}$ such that $\sum_{i = 1}^q a_i = O(\sqrt q)$.
    Such a network, defined by $W$ and $a$, computes the function:
    \[
    h_{W, a}(x) = \frac{1}{\sqrt{q}}\sum_{i = 1}^q a_i \phi(\langle w_i, x\rangle ).
    \]
    For our proof we consider $\phi$ to be the absolute value function, and later show how to obtain an MLP with ReLU activations.
    
     Consider the set of samples $((x_1, y_1), \ldots, (x_N, y_N))$, where $x_i \sim \cN(0, I_M)$ and $y_i \in \{\pm 1\}$ is a Rademacher random variable. The objective is to find a $W^*, a^*$ such that $y_i h_{W^*, a^*}(x_i) > 0 \ \forall i$. \cite{daniely2020memorizing} gives the following result: for $N \leq \frac{Mq}{\log^4(M)}$ and $q \geq \log^4(M)$, with probability $1 - o(1)$ there exists $h$ such that, for every $i \in [N]$,  $y_ih(x) = \Omega(\ln(M))$.

     We build on this result as follows: first, we assume that the attention matrix performs averaging of the input sequence, such that $x_i = \frac{1}{N}\sum_{j = 1}^N (X^i)_j$ where  where $X^i$ is the $i$-th input sequence to memorize. Assuming the inputs are also Gaussian, then so is the average. Hence, $x_i \in \Rset^m$ and $x_i \sim \cN(0, cI_m)$ for $c = \frac{1}{L}$.
    
    Next, we fix the number of experts to be $K$. We fix a router function $r: \Rset^{\log K} \rightarrow [K]$ such that, on average, each expert must memorize $n/K$ sequences, and with high probability, each expert must memorize at most $2n/K$ sequences. We use the construction from Lemma \ref{lemma:prob1k} where the router only acts on the first $\log K$ entries of each vector.
    Hence, we have $K$ MLPs, each tasked with memorizing at most $\frac{2n}{K}$ input-output values. We use the network construction from \cite{daniely2020memorizing} to only the last $m-\log K \ge m/2$ coordinates, which are not used by the router, so that their distribution is independent and Gaussian. 
    
    Hence, with $n \leq \frac{K}{2}\frac{(m - \log K) q}{\log^4(m - \log K))} \leq \frac{K}{4}\frac{m q}{\log^4(m - \log K))}$ and $q \geq \log^4(m - \log K)$, with high probability, $y_ih_j(x_i) = \Omega(\ln(m))$, where $h_j$ is the expert to which $x_i$ is routed.  It follows that width 
    \[
    q \geq \frac{4n \log^4 (m - \log K)}{m K}
    \]
    is sufficient for each individual expert, and we assume that $\frac{4n}{mK} \geq 1$.

    In order to obtain a MLP with a ReLU activation defined as $\sigma(x) = \max\{0, x\}$, we use the fact that $\abs{x} = \sigma(x) + \sigma(-x)$. This is because absolute value is a valid activation function in \cite{daniely2020memorizing}, but ReLU is not. However, doubling the width of each MLP is sufficient to obtain MLPs with ReLU activations instead, i.e.:
    \[
    q \geq \frac{8n \log^4 (m - \log K)}{mK}.
    \]
    This is done by taking a solution which is of form:
    \[
    h_{W, a}(x) = \frac{1}{\sqrt{q}}\sum_{i = 1}^q a_i \abs{\langle w_i, x\rangle},
    \]
    and converting it to the form:
    \[
    h_{W, a}^{'}(x) = \frac{1}{\sqrt{q}}\sum_{i = 1}^q a_i\paren*{ (\sigma{\langle w_i, x\rangle}) + (\sigma{\langle - w_i, x\rangle})}.
    \]

    The router has $K$ vectors each of dimension $m$. Consequently, we need $O(K m)$ parameters to store the router. Each expert has width $\frac{8n \log^4(m - \log K)}{mK} = \tilde{O}\paren*{\frac{n}{mK}}$ and $\tilde{O}\paren*{\frac{n}{K}}$ parameters. 
\end{proof}

\begin{corollary}
   
    Let $\delta \in (0,1)$, and fix $K > 1$. For $n \geq \frac{K^2 \ln (K / \delta)}{2}$,  with probability  at least $1-\delta$, there exists a sparse transformer $s$ with $K$ experts such that $y_i s(x_i) = \Omega(\ln m)$. It has $\tilde{O}(n + Km)$ parameters and  $\tilde{O}(n/K + Km)$ active parameters.
\end{corollary}
\begin{proof}
    For each expert, we apply the result from Lemma C.8 with $\delta^\prime = \delta / K$. Hence, for every expert, with probability at most $\delta / K$, there does not exist an MLP of width $\tilde{O}(n / mK)$ which memorizes its samples. 
    By a union bound, this implies that with probability at most $\delta$, at least one of the experts is not able to memorize its samples. Hence, with probability at least $1 - \delta$, all of the experts are able to perform memorization on their given samples. In total, this implies we will use $\tilde{O}(n + K)$ parameters to store the entire mixture of expert network. The number of active parameters is $\tilde{O}(n/K + K)$.
\end{proof}

\begin{lemma}[Bound on $\ell_2$-norm of vector from $\cN(0, c I_m)$ \label{lemma:l2bound}]
    Let $x \sim \cN(0, c I_m)$ for some $c > 0$, and let $\delta \in (0, 1)$. Then, with probability at least $1 - \delta$, there exists a constant $c_{\delta,m} > 0$, such that
    \[
    \norm{x}_2 \leq c_{\delta,m} = \sqrt{ c \left( m + 2 \sqrt{ m \ln\left( \dfrac{1}{\delta} \right) } + 2 \ln\left( \dfrac{1}{\delta} \right) \right) }.
    \]
\end{lemma}
\begin{proof}
    Each component $x_i$ of $x$ is distributed as $x_i \sim \cN(0, c)$. We can express $x_i$ as $x_i = \sqrt{c}\, z_i$, where $z_i \sim \cN(0, 1)$. Then,
    \[
    \norm{x}_2^2 = \sum_{i=1}^m x_i^2 = c \sum_{i=1}^m z_i^2.
    \]
    The sum $\sum_{i=1}^m z_i^2$ follows a chi-squared distribution with $m$ degrees of freedom. By the Laurent-Massart theorem~\cite{73a7bee2-f91d-3c13-8f3a-95b04baa6599}, for all $t > 0$, we have
    \[
    \Pr\left( \sum_{i=1}^m z_i^2 \geq m + 2 \sqrt{m t} + 2 t \right) \leq e^{-t}.
    \]
    Multiplying both sides inside the probability by $c$, we obtain
    \[
    \Pr\left( \norm{x}_2^2 \geq c \left( m + 2 \sqrt{m t} + 2 t \right) \right) \leq e^{-t}.
    \]
    Setting $t = \ln\left( \dfrac{1}{\delta} \right)$, it follows that
    \[
    \Pr\left( \norm{x}_2^2 \geq c \left( m + 2 \sqrt{m \ln\left( \dfrac{1}{\delta} \right)} + 2 \ln\left( \dfrac{1}{\delta} \right) \right) \right) \leq \delta.
    \]
    Therefore, with probability at least $1 - \delta$, we have
    \[
    \norm{x}_2 \leq \sqrt{ c \left( m + 2 \sqrt{ m \ln\left( \dfrac{1}{\delta} \right) } + 2 \ln\left( \dfrac{1}{\delta} \right) \right) }.
    \]
\end{proof}

\begin{lemma}[Bounded bit complexity required]
With high probability, $\tilde{O}(1)$ bits are required to store each weight in the network, and $\tilde O (n+K)$ bits are required to store the entire network. $\tilde{O}\paren*{\frac{n}{K}+K}$ active bits are required.
\end{lemma}
\begin{proof}
We show that the learning result from \cite{daniely2020memorizing} gives a network with bounded bit complexity. Hence, it suffices to bound the bit complexity of the initial weights and the bit complexity of the gradient update.
First, we begin with $h(x) = \sqrt{\frac{1}{q}} \sum_{i = 1}^q a_i \sigma(\langle w_i, x\rangle)$. The objective is to show there exists $\tilde{h}(x)$ with bounded bit complexity which satisfies $y_i \tilde{h}(x_i) > 0 \ \forall i$. Suppose we begin by creating bins for the weights, and replace each $w_i$ with $\tilde{w}_i$ where $\norm{w_i - \tilde w_i}_\infty \leq \varepsilon$. Further, assume that $\norm{w_i}_\infty \leq M$ for some $M > 0$. We then obtain that $\norm{w_i - \tilde w_i}_2 \leq \sqrt{m} \varepsilon$. 
In addition, we state that with probability at least $1 - \delta$, $\norm{x}_2 \leq c_{\delta,m}$ due to Lemma \ref{lemma:l2bound}.
Using this, we have that:
\begin{align*}
    \abs{\sigma(\langle w_i, x \rangle) - \sigma(\langle \tilde w_i, x \rangle)} &\leq \abs{\langle w_i - \tilde w_i, x \rangle }  && (\sigma \text{ is } 1 \text{-Lipschitz}) \\
    &\leq \norm{w_i - \tilde w_i}_2 \norm{x}_2 &&(\text{Cauchy-Schwarz})\\
    &\leq \sqrt m \varepsilon c_{\delta,m} \text{ w.p.} \geq 1 - \delta&& (\text{Lemma \ref{lemma:l2bound}}).
\end{align*}
We apply this result to obtain that with probability at least $1 - \delta$,
\begin{align*}
    \abs{h(x) - \tilde h (x)} &\leq \frac{1}{\sqrt q} \sum_{i = 1}^q \abs{a_i} \abs*{\sigma(\langle w_i, x \rangle) - \sigma(\langle \tilde w_i, x \rangle)} \\
    &\leq \sqrt{q} \sqrt m \varepsilon c_{\delta,m}.
\end{align*}
Given that $h(x) \geq O(1)$, then for some arbitrarily small constant (we use $\frac{1}{4}$) we require: $\frac{1}{4} \leq \sqrt{q} \sqrt m \varepsilon c_{\delta,m}$, or equivalently, $\varepsilon \leq \frac{1}{4c_{\delta,m}\sqrt{qm}}$. Because $h(x) \geq O(1)$, we use this separation of $\frac{1}{4}$ to show that $\tilde{h}(x)$ remains the same sign as $h(x)$ (however this constant can be replaced with an arbitrarily small constant so as to satisfy the requirement). 

Consider $w_i^{(0)}$ to be the initialization of $w_i$ prior to the gradient step. Then, because, $M = \norm{w_i}_\infty$, we have that $M = \max_i \curl*{{w_i^{(0)}} - \eta \frac{\partial h(x)}{\partial w_i}} = \max_i \curl*{{w_i^{(0)}} - \frac{n \log m}{m} \frac{\partial h(x)}{\partial w_i}}$. Assuming that $w_i^{(0)}$ is initialized with bounded $\ell_\infty$ norm of $1$, then we obtain that $M \leq 1 + \frac{n \log m}{m \sqrt q }c_{\delta,m}$.
This is due to Lemma \ref{lemma:l2bound} and that
\begin{align*}
\frac{\partial h(x)}{\partial w_i} &= \frac{1}{\sqrt q} a_i \sigma' (\langle w_i, x_i \rangle) x_i \\
&= \begin{cases}
    0 & \text{ if }\langle w_i, x_i \rangle \leq 0  \\
    \frac{1}{\sqrt q} a_i  x_i & \text{ otherwise}.
\end{cases}
\end{align*}
Using this, with high probability, we require $\log \paren*{\frac{2M}{\varepsilon}} = \log\paren*{8M c_\delta \sqrt{qm}}$  bits (by replacing $\varepsilon$). Hence, with high probability, we require the number of bits to be at most:
\[
O\paren*{\log\paren*{ \frac{n \log m}{\sqrt m }c_{\delta,m}^2}} = \tilde{O}(1).
\] 
\end{proof}

We restate Theorem \ref{thm:lower_bound_memorization}.
\begin{theorem}[Lower bound for dense model]
    Given the same task as above, a dense Transformer requires $\tilde{\Omega}(n)$ parameters to solve the memorization task.
\end{theorem}

\begin{proof}[Proof of Theorem \ref{thm:lower_bound_memorization}]
    Let $c$ be the number of bits used for encoding each parameters (and we assume that $c$ is logarithmic in the problem parameters).
    Denote by $\cH$ the class of all transformers with $W$ parameters and $c$ bits per parameters. Since $\cH$ is a finite class, where each function in the class can be encoded with $cW$ bits, we have $\abs{\cH} \le 2^{cW}$. 
    Let $X^1, \dots, X^N \in \mathbb{R}^{n \times d}$ be the $N$ input points. Assume a $\cH$ can solve the memorization task. Then, for every choice of $y_1, \dots, y_N \in \{\pm 1\}$, there exists a transformer $f \in \cH$ s.t. $f(X_i) = y_i$ for all $i \in [N]$. There are $2^N$ possible assignments for $y_1, \dots y_N$ and therefore there are at least $2^{N}$ different functions in $\cH$. So, we get $2^N \le \abs{\cH} \le 2^{cW}$ and therefore $W \ge N/c$.
\end{proof}

\begin{lemma}[Active parameter comparison between dense and sparse transformers]
There exist cases in which the amount of active parameters required to perform memorization is less for a sparse transformer than a dense transformer.
\end{lemma}
\begin{proof}
    As shown, a dense transformer requires $\tilde{\Omega}(n)$ parameters (and active parameters) to perform memorization of $N$ sequences. In contrast, for $n$ sufficiently large and fixed $K>1$, it holds that $\frac{n}{K} +K <n$, which shows that the number of active parameters required is less for a sparse transformer.
\end{proof}

%% file: main.bbl
\begin{thebibliography}{106}
\providecommand{\natexlab}[1]{#1}
\providecommand{\url}[1]{\texttt{#1}}
\expandafter\ifx\csname urlstyle\endcsname\relax
  \providecommand{\doi}[1]{doi: #1}\else
  \providecommand{\doi}{doi: \begingroup \urlstyle{rm}\Url}\fi

\bibitem[Abbe et~al.(2024)Abbe, Bengio, Lotfi, Sandon, and
  Saremi]{abbe2024fartransformersreasonlocality}
Emmanuel Abbe, Samy Bengio, Aryo Lotfi, Colin Sandon, and Omid Saremi.
\newblock How far can transformers reason? the locality barrier and inductive
  scratchpad, 2024.
\newblock URL \url{https://arxiv.org/abs/2406.06467}.

\bibitem[Achiam et~al.(2023)Achiam, Adler, Agarwal, Ahmad, Akkaya, Aleman,
  Almeida, Altenschmidt, Altman, Anadkat, et~al.]{achiam2023gpt}
Josh Achiam, Steven Adler, Sandhini Agarwal, Lama Ahmad, Ilge Akkaya,
  Florencia~Leoni Aleman, Diogo Almeida, Janko Altenschmidt, Sam Altman,
  Shyamal Anadkat, et~al.
\newblock Gpt-4 technical report.
\newblock \emph{arXiv preprint arXiv:2303.08774}, 2023.

\bibitem[Allen-Zhu \& Li(2023)Allen-Zhu and Li]{allen2023physics}
Zeyuan Allen-Zhu and Yuanzhi Li.
\newblock Physics of language models: Part 3.1, knowledge storage and
  extraction.
\newblock \emph{arXiv preprint arXiv:2309.14316}, 2023.

\bibitem[Anil et~al.(2023)Anil, Borgeaud, Wu, Alayrac, Yu, Soricut, Schalkwyk,
  Dai, Hauth, Millican, et~al.]{anil2023gemini}
Rohan Anil, Sebastian Borgeaud, Yonghui Wu, Jean-Baptiste Alayrac, Jiahui Yu,
  Radu Soricut, Johan Schalkwyk, Andrew~M Dai, Anja Hauth, Katie Millican,
  et~al.
\newblock Gemini: A family of highly capable multimodal models.
\newblock \emph{arXiv preprint arXiv:2312.11805}, 1, 2023.

\bibitem[Antoniak et~al.(2023)Antoniak, Jaszczur, Krutul, Pi{\'o}ro, Krajewski,
  Ludziejewski, Odrzyg{\'o}{\'z}d{\'z}, and Cygan]{antoniak2023mixture}
Szymon Antoniak, Sebastian Jaszczur, Micha{\l} Krutul, Maciej Pi{\'o}ro, Jakub
  Krajewski, Jan Ludziejewski, Tomasz Odrzyg{\'o}{\'z}d{\'z}, and Marek Cygan.
\newblock Mixture of tokens: Efficient llms through cross-example aggregation.
\newblock \emph{arXiv preprint arXiv:2310.15961}, 2023.

\bibitem[Artetxe et~al.(2021)Artetxe, Bhosale, Goyal, Mihaylov, Ott, Shleifer,
  Lin, Du, Iyer, Pasunuru, et~al.]{artetxe2021efficient}
Mikel Artetxe, Shruti Bhosale, Naman Goyal, Todor Mihaylov, Myle Ott, Sam
  Shleifer, Xi~Victoria Lin, Jingfei Du, Srinivasan Iyer, Ramakanth Pasunuru,
  et~al.
\newblock Efficient large scale language modeling with mixtures of experts.
\newblock \emph{arXiv preprint arXiv:2112.10684}, 2021.

\bibitem[Azerbayev et~al.(2023)Azerbayev, Schoelkopf, Paster, Santos, McAleer,
  Jiang, Deng, Biderman, and Welleck]{azerbayev2023llemma}
Zhangir Azerbayev, Hailey Schoelkopf, Keiran Paster, Marco~Dos Santos, Stephen
  McAleer, Albert~Q. Jiang, Jia Deng, Stella Biderman, and Sean Welleck.
\newblock Llemma: An open language model for mathematics, 2023.

\bibitem[Ben~Allal et~al.(2024)Ben~Allal, Lozhkov, Penedo, Wolf, and von
  Werra]{benallal2024cosmopedia}
Loubna Ben~Allal, Anton Lozhkov, Guilherme Penedo, Thomas Wolf, and Leandro von
  Werra.
\newblock Cosmopedia, 2024.
\newblock URL \url{https://huggingface.co/datasets/HuggingFaceTB/cosmopedia}.

\bibitem[Berant et~al.(2013)Berant, Chou, Frostig, and
  Liang]{berant2013semantic}
Jonathan Berant, Andrew Chou, Roy Frostig, and Percy Liang.
\newblock Semantic parsing on freebase from question-answer pairs.
\newblock In \emph{Proceedings of the 2013 conference on empirical methods in
  natural language processing}, pp.\  1533--1544, 2013.

\bibitem[Bisk et~al.(2020)Bisk, Zellers, Gao, Choi, et~al.]{bisk2020piqa}
Yonatan Bisk, Rowan Zellers, Jianfeng Gao, Yejin Choi, et~al.
\newblock Piqa: Reasoning about physical commonsense in natural language.
\newblock In \emph{Proceedings of the AAAI conference on artificial
  intelligence}, volume~34, pp.\  7432--7439, 2020.

\bibitem[Chen \& Zou(2024)Chen and Zou]{chen2024can}
Xingwu Chen and Difan Zou.
\newblock What can transformer learn with varying depth? case studies on
  sequence learning tasks.
\newblock \emph{arXiv preprint arXiv:2404.01601}, 2024.

\bibitem[Chen et~al.(2022)Chen, Deng, Wu, Gu, and Li]{chen2022towards}
Zixiang Chen, Yihe Deng, Yue Wu, Quanquan Gu, and Yuanzhi Li.
\newblock Towards understanding mixture of experts in deep learning.
\newblock \emph{arXiv preprint arXiv:2208.02813}, 2022.

\bibitem[Cho et~al.(2024)Cho, Cha, Awasthi, Bhojanapalli, Gupta, and
  Yun]{cho2024position}
Hanseul Cho, Jaeyoung Cha, Pranjal Awasthi, Srinadh Bhojanapalli, Anupam Gupta,
  and Chulhee Yun.
\newblock Position coupling: Leveraging task structure for improved length
  generalization of transformers.
\newblock \emph{arXiv preprint arXiv:2405.20671}, 2024.

\bibitem[Clark et~al.(2022)Clark, de~Las~Casas, Guy, Mensch, Paganini,
  Hoffmann, Damoc, Hechtman, Cai, Borgeaud, et~al.]{clark2022unified}
Aidan Clark, Diego de~Las~Casas, Aurelia Guy, Arthur Mensch, Michela Paganini,
  Jordan Hoffmann, Bogdan Damoc, Blake Hechtman, Trevor Cai, Sebastian
  Borgeaud, et~al.
\newblock Unified scaling laws for routed language models.
\newblock In \emph{International conference on machine learning}, pp.\
  4057--4086. PMLR, 2022.

\bibitem[Clark et~al.(2018)Clark, Cowhey, Etzioni, Khot, Sabharwal, Schoenick,
  and Tafjord]{clark2018think}
Peter Clark, Isaac Cowhey, Oren Etzioni, Tushar Khot, Ashish Sabharwal, Carissa
  Schoenick, and Oyvind Tafjord.
\newblock Think you have solved question answering? try arc, the ai2 reasoning
  challenge.
\newblock \emph{arXiv preprint arXiv:1803.05457}, 2018.

\bibitem[Cobbe et~al.(2021)Cobbe, Kosaraju, Bavarian, Chen, Jun, Kaiser,
  Plappert, Tworek, Hilton, Nakano, et~al.]{cobbe2021training}
Karl Cobbe, Vineet Kosaraju, Mohammad Bavarian, Mark Chen, Heewoo Jun, Lukasz
  Kaiser, Matthias Plappert, Jerry Tworek, Jacob Hilton, Reiichiro Nakano,
  et~al.
\newblock Training verifiers to solve math word problems.
\newblock \emph{arXiv preprint arXiv:2110.14168}, 2021.

\bibitem[Dai et~al.(2024)Dai, Deng, Zhao, Xu, Gao, Chen, Li, Zeng, Yu, Wu,
  et~al.]{dai2024deepseekmoe}
Damai Dai, Chengqi Deng, Chenggang Zhao, RX~Xu, Huazuo Gao, Deli Chen, Jiashi
  Li, Wangding Zeng, Xingkai Yu, Y~Wu, et~al.
\newblock Deepseekmoe: Towards ultimate expert specialization in
  mixture-of-experts language models.
\newblock \emph{arXiv preprint arXiv:2401.06066}, 2024.

\bibitem[Daniely(2020)]{daniely2020memorizing}
Amit Daniely.
\newblock Memorizing gaussians with no over-parameterizaion via gradient decent
  on neural networks.
\newblock \emph{arXiv preprint arXiv:2003.12895}, 2020.

\bibitem[Dao et~al.(2021)Dao, Chen, Liang, Yang, Song, Rudra, and
  Re]{dao2021pixelated}
Tri Dao, Beidi Chen, Kaizhao Liang, Jiaming Yang, Zhao Song, Atri Rudra, and
  Christopher Re.
\newblock Pixelated butterfly: Simple and efficient sparse training for neural
  network models.
\newblock \emph{arXiv preprint arXiv:2112.00029}, 2021.

\bibitem[Dao et~al.(2022)Dao, Chen, Sohoni, Desai, Poli, Grogan, Liu, Rao,
  Rudra, and R{\'e}]{dao2022monarch}
Tri Dao, Beidi Chen, Nimit~S Sohoni, Arjun Desai, Michael Poli, Jessica Grogan,
  Alexander Liu, Aniruddh Rao, Atri Rudra, and Christopher R{\'e}.
\newblock Monarch: Expressive structured matrices for efficient and accurate
  training.
\newblock In \emph{International Conference on Machine Learning}, pp.\
  4690--4721. PMLR, 2022.

\bibitem[Databricks(2023)]{databricks2023}
Databricks.
\newblock Introducing dbrx: A new state-of-the-art open llm.
\newblock \emph{Databricks Blog}, 2023.
\newblock URL
  \url{https://www.databricks.com/blog/introducing-dbrx-new-state-art-open-llm}.
\newblock Accessed: 2023-10-12.

\bibitem[DeepMind(2024)]{deepmind2024ai}
DeepMind.
\newblock Ai achieves silver-medal standard solving international mathematical
  olympiad problems.
\newblock
  \url{https://deepmind.google/discover/blog/ai-solves-imo-problems-at-silver-medal-level/},
  2024.

\bibitem[Du et~al.(2022)Du, Huang, Dai, Tong, Lepikhin, Xu, Krikun, Zhou, Yu,
  Firat, et~al.]{du2022glam}
Nan Du, Yanping Huang, Andrew~M Dai, Simon Tong, Dmitry Lepikhin, Yuanzhong Xu,
  Maxim Krikun, Yanqi Zhou, Adams~Wei Yu, Orhan Firat, et~al.
\newblock Glam: Efficient scaling of language models with mixture-of-experts.
\newblock In \emph{International Conference on Machine Learning}, pp.\
  5547--5569. PMLR, 2022.

\bibitem[Fatemi et~al.(2023)Fatemi, Halcrow, and Perozzi]{fatemi2023talk}
Bahare Fatemi, Jonathan Halcrow, and Bryan Perozzi.
\newblock Talk like a graph: Encoding graphs for large language models.
\newblock \emph{arXiv preprint arXiv:2310.04560}, 2023.

\bibitem[Fedus et~al.(2022)Fedus, Zoph, and Shazeer]{fedus2022switch}
William Fedus, Barret Zoph, and Noam Shazeer.
\newblock Switch transformers: Scaling to trillion parameter models with simple
  and efficient sparsity.
\newblock \emph{Journal of Machine Learning Research}, 23\penalty0
  (120):\penalty0 1--39, 2022.

\bibitem[Fu et~al.(2024)Fu, Arora, Grogan, Johnson, Eyuboglu, Thomas, Spector,
  Poli, Rudra, and R{\'e}]{fu2024monarch}
Dan Fu, Simran Arora, Jessica Grogan, Isys Johnson, Evan~Sabri Eyuboglu, Armin
  Thomas, Benjamin Spector, Michael Poli, Atri Rudra, and Christopher R{\'e}.
\newblock Monarch mixer: A simple sub-quadratic gemm-based architecture.
\newblock \emph{Advances in Neural Information Processing Systems}, 36, 2024.

\bibitem[Fu et~al.(2022)Fu, Dao, Saab, Thomas, Rudra, and R{\'e}]{fu2022hungry}
Daniel~Y Fu, Tri Dao, Khaled~K Saab, Armin~W Thomas, Atri Rudra, and
  Christopher R{\'e}.
\newblock Hungry hungry hippos: Towards language modeling with state space
  models.
\newblock \emph{arXiv preprint arXiv:2212.14052}, 2022.

\bibitem[Gale et~al.(2023)Gale, Narayanan, Young, and
  Zaharia]{gale2023megablocks}
Trevor Gale, Deepak Narayanan, Cliff Young, and Matei Zaharia.
\newblock Megablocks: Efficient sparse training with mixture-of-experts.
\newblock \emph{Proceedings of Machine Learning and Systems}, 5:\penalty0
  288--304, 2023.

\bibitem[Gao et~al.(2023)Gao, Madaan, Zhou, Alon, Liu, Yang, Callan, and
  Neubig]{gao2023pal}
Luyu Gao, Aman Madaan, Shuyan Zhou, Uri Alon, Pengfei Liu, Yiming Yang, Jamie
  Callan, and Graham Neubig.
\newblock Pal: Program-aided language models.
\newblock In \emph{International Conference on Machine Learning}, pp.\
  10764--10799. PMLR, 2023.

\bibitem[Gu \& Dao(2023)Gu and Dao]{gu2023mamba}
Albert Gu and Tri Dao.
\newblock Mamba: Linear-time sequence modeling with selective state spaces.
\newblock \emph{arXiv preprint arXiv:2312.00752}, 2023.

\bibitem[Heinzerling \& Inui(2020)Heinzerling and
  Inui]{heinzerling2020language}
Benjamin Heinzerling and Kentaro Inui.
\newblock Language models as knowledge bases: On entity representations,
  storage capacity, and paraphrased queries.
\newblock \emph{arXiv preprint arXiv:2008.09036}, 2020.

\bibitem[Hendrycks et~al.(2021)Hendrycks, Burns, Kadavath, Arora, Basart, Tang,
  Song, and Steinhardt]{hendrycks2021measuring}
Dan Hendrycks, Collin Burns, Saurav Kadavath, Akul Arora, Steven Basart, Eric
  Tang, Dawn Song, and Jacob Steinhardt.
\newblock Measuring mathematical problem solving with the math dataset.
\newblock \emph{arXiv preprint arXiv:2103.03874}, 2021.

\bibitem[Hou et~al.(2024)Hou, Brandfonbrener, Kakade, Jelassi, and
  Malach]{hou2024universal}
Kaiying Hou, David Brandfonbrener, Sham Kakade, Samy Jelassi, and Eran Malach.
\newblock Universal length generalization with turing programs.
\newblock \emph{arXiv preprint arXiv:2407.03310}, 2024.

\bibitem[Imani et~al.(2023)Imani, Du, and Shrivastava]{imani2023mathprompter}
Shima Imani, Liang Du, and Harsh Shrivastava.
\newblock Mathprompter: Mathematical reasoning using large language models.
\newblock \emph{arXiv preprint arXiv:2303.05398}, 2023.

\bibitem[Jacobs et~al.(1991)Jacobs, Jordan, Nowlan, and
  Hinton]{jacobs1991adaptive}
Robert~A Jacobs, Michael~I Jordan, Steven~J Nowlan, and Geoffrey~E Hinton.
\newblock Adaptive mixtures of local experts.
\newblock \emph{Neural computation}, 3\penalty0 (1):\penalty0 79--87, 1991.

\bibitem[Jelassi et~al.(2023)Jelassi, d'Ascoli, Domingo-Enrich, Wu, Li, and
  Charton]{jelassi2023length}
Samy Jelassi, St{\'e}phane d'Ascoli, Carles Domingo-Enrich, Yuhuai Wu, Yuanzhi
  Li, and Fran{\c{c}}ois Charton.
\newblock Length generalization in arithmetic transformers.
\newblock \emph{arXiv preprint arXiv:2306.15400}, 2023.

\bibitem[Jiang et~al.(2023)Jiang, Sablayrolles, Mensch, Bamford, Chaplot,
  Casas, Bressand, Lengyel, Lample, Saulnier, et~al.]{jiang2023mistral}
Albert~Q Jiang, Alexandre Sablayrolles, Arthur Mensch, Chris Bamford,
  Devendra~Singh Chaplot, Diego de~las Casas, Florian Bressand, Gianna Lengyel,
  Guillaume Lample, Lucile Saulnier, et~al.
\newblock Mistral 7b.
\newblock \emph{arXiv preprint arXiv:2310.06825}, 2023.

\bibitem[Jiang et~al.(2024)Jiang, Sablayrolles, Roux, Mensch, Savary, Bamford,
  Chaplot, Casas, Hanna, Bressand, et~al.]{jiang2024mixtral}
Albert~Q Jiang, Alexandre Sablayrolles, Antoine Roux, Arthur Mensch, Blanche
  Savary, Chris Bamford, Devendra~Singh Chaplot, Diego de~las Casas, Emma~Bou
  Hanna, Florian Bressand, et~al.
\newblock Mixtral of experts.
\newblock \emph{arXiv preprint arXiv:2401.04088}, 2024.

\bibitem[Jin et~al.(2023)Jin, Liu, Han, Jiang, Ji, and Han]{jin2023large}
Bowen Jin, Gang Liu, Chi Han, Meng Jiang, Heng Ji, and Jiawei Han.
\newblock Large language models on graphs: A comprehensive survey.
\newblock \emph{arXiv preprint arXiv:2312.02783}, 2023.

\bibitem[Jordan \& Jacobs(1994)Jordan and Jacobs]{jordan1994hierarchical}
Michael~I Jordan and Robert~A Jacobs.
\newblock Hierarchical mixtures of experts and the em algorithm.
\newblock \emph{Neural computation}, 6\penalty0 (2):\penalty0 181--214, 1994.

\bibitem[Joshi et~al.(2017)Joshi, Choi, Weld, and
  Zettlemoyer]{joshi2017triviaqa}
Mandar Joshi, Eunsol Choi, Daniel~S Weld, and Luke Zettlemoyer.
\newblock Triviaqa: A large scale distantly supervised challenge dataset for
  reading comprehension.
\newblock \emph{arXiv preprint arXiv:1705.03551}, 2017.

\bibitem[Kamalakara et~al.(2022)Kamalakara, Locatelli, Venkitesh, Ba, Gal, and
  Gomez]{kamalakara2022exploring}
Siddhartha~Rao Kamalakara, Acyr Locatelli, Bharat Venkitesh, Jimmy Ba, Yarin
  Gal, and Aidan~N Gomez.
\newblock Exploring low rank training of deep neural networks.
\newblock \emph{arXiv preprint arXiv:2209.13569}, 2022.

\bibitem[Katharopoulos et~al.(2020)Katharopoulos, Vyas, Pappas, and
  Fleuret]{katharopoulos2020transformers}
Angelos Katharopoulos, Apoorv Vyas, Nikolaos Pappas, and Fran{\c{c}}ois
  Fleuret.
\newblock Transformers are rnns: Fast autoregressive transformers with linear
  attention.
\newblock In \emph{International conference on machine learning}, pp.\
  5156--5165. PMLR, 2020.

\bibitem[Kim et~al.(2023)Kim, Kim, and Mozafari]{kim2023provable}
Junghwan Kim, Michelle Kim, and Barzan Mozafari.
\newblock Provable memorization capacity of transformers.
\newblock In \emph{The Eleventh International Conference on Learning
  Representations}, 2023.

\bibitem[Krajewski et~al.(2024)Krajewski, Ludziejewski, Adamczewski, Pi{\'o}ro,
  Krutul, Antoniak, Ciebiera, Kr{\'o}l, Odrzyg{\'o}{\'z}d{\'z}, Sankowski,
  et~al.]{krajewski2024scaling}
Jakub Krajewski, Jan Ludziejewski, Kamil Adamczewski, Maciej Pi{\'o}ro,
  Micha{\l} Krutul, Szymon Antoniak, Kamil Ciebiera, Krystian Kr{\'o}l, Tomasz
  Odrzyg{\'o}{\'z}d{\'z}, Piotr Sankowski, et~al.
\newblock Scaling laws for fine-grained mixture of experts.
\newblock \emph{arXiv preprint arXiv:2402.07871}, 2024.

\bibitem[Kwiatkowski et~al.(2019)Kwiatkowski, Palomaki, Redfield, Collins,
  Parikh, Alberti, Epstein, Polosukhin, Devlin, Lee,
  et~al.]{kwiatkowski2019natural}
Tom Kwiatkowski, Jennimaria Palomaki, Olivia Redfield, Michael Collins, Ankur
  Parikh, Chris Alberti, Danielle Epstein, Illia Polosukhin, Jacob Devlin,
  Kenton Lee, et~al.
\newblock Natural questions: a benchmark for question answering research.
\newblock \emph{Transactions of the Association for Computational Linguistics},
  7:\penalty0 453--466, 2019.

\bibitem[Laurent \& Massart(2000)Laurent and
  Massart]{73a7bee2-f91d-3c13-8f3a-95b04baa6599}
B.~Laurent and P.~Massart.
\newblock Adaptive estimation of a quadratic functional by model selection.
\newblock \emph{The Annals of Statistics}, 28\penalty0 (5):\penalty0
  1302--1338, 2000.
\newblock ISSN 00905364, 21688966.
\newblock URL \url{http://www.jstor.org/stable/2674095}.

\bibitem[Lee et~al.(2023)Lee, Sreenivasan, Lee, Lee, and
  Papailiopoulos]{lee2023teaching}
Nayoung Lee, Kartik Sreenivasan, Jason~D Lee, Kangwook Lee, and Dimitris
  Papailiopoulos.
\newblock Teaching arithmetic to small transformers.
\newblock \emph{arXiv preprint arXiv:2307.03381}, 2023.

\bibitem[Lepikhin et~al.(2020)Lepikhin, Lee, Xu, Chen, Firat, Huang, Krikun,
  Shazeer, and Chen]{lepikhin2020gshard}
Dmitry Lepikhin, HyoukJoong Lee, Yuanzhong Xu, Dehao Chen, Orhan Firat, Yanping
  Huang, Maxim Krikun, Noam Shazeer, and Zhifeng Chen.
\newblock Gshard: Scaling giant models with conditional computation and
  automatic sharding.
\newblock \emph{arXiv preprint arXiv:2006.16668}, 2020.

\bibitem[Lewis et~al.(2021)Lewis, Bhosale, Dettmers, Goyal, and
  Zettlemoyer]{lewis2021base}
Mike Lewis, Shruti Bhosale, Tim Dettmers, Naman Goyal, and Luke Zettlemoyer.
\newblock Base layers: Simplifying training of large, sparse models.
\newblock In \emph{International Conference on Machine Learning}, pp.\
  6265--6274. PMLR, 2021.

\bibitem[Lewkowycz et~al.(2022)Lewkowycz, Andreassen, Dohan, Dyer, Michalewski,
  Ramasesh, Slone, Anil, Schlag, Gutman-Solo, et~al.]{lewkowycz2022solving}
Aitor Lewkowycz, Anders Andreassen, David Dohan, Ethan Dyer, Henryk
  Michalewski, Vinay Ramasesh, Ambrose Slone, Cem Anil, Imanol Schlag, Theo
  Gutman-Solo, et~al.
\newblock Solving quantitative reasoning problems with language models.
\newblock \emph{Advances in Neural Information Processing Systems},
  35:\penalty0 3843--3857, 2022.

\bibitem[Liu et~al.(2023)Liu, Bubeck, Eldan, Kulkarni, Li, Nguyen, Ward, and
  Zhang]{liu2023tinygsm}
Bingbin Liu, Sebastien Bubeck, Ronen Eldan, Janardhan Kulkarni, Yuanzhi Li, Anh
  Nguyen, Rachel Ward, and Yi~Zhang.
\newblock Tinygsm: achieving> 80\% on gsm8k with small language models.
\newblock \emph{arXiv preprint arXiv:2312.09241}, 2023.

\bibitem[Liu et~al.(2024)Liu, Liu, Chen, Chen, Zan, Kan, and Ho]{liu2024devil}
Yan Liu, Yu~Liu, Xiaokang Chen, Pin-Yu Chen, Daoguang Zan, Min-Yen Kan, and
  Tsung-Yi Ho.
\newblock The devil is in the neurons: Interpreting and mitigating social
  biases in language models.
\newblock In \emph{The Twelfth International Conference on Learning
  Representations}, 2024.

\bibitem[Loshchilov et~al.(2017)Loshchilov, Hutter,
  et~al.]{loshchilov2017fixing}
Ilya Loshchilov, Frank Hutter, et~al.
\newblock Fixing weight decay regularization in adam.
\newblock \emph{arXiv preprint arXiv:1711.05101}, 5, 2017.

\bibitem[Madden et~al.(2024)Madden, Fox, and Thrampoulidis]{madden2024upper}
Liam Madden, Curtis Fox, and Christos Thrampoulidis.
\newblock Upper and lower memory capacity bounds of transformers for next-token
  prediction.
\newblock \emph{arXiv preprint arXiv:2405.13718}, 2024.

\bibitem[Mahdavi et~al.(2023)Mahdavi, Liao, and
  Thrampoulidis]{mahdavi2023memorization}
Sadegh Mahdavi, Renjie Liao, and Christos Thrampoulidis.
\newblock Memorization capacity of multi-head attention in transformers.
\newblock \emph{arXiv preprint arXiv:2306.02010}, 2023.

\bibitem[Malach(2024)]{malach2023auto}
Eran Malach.
\newblock Auto-regressive next-token predictors are universal learners.
\newblock In \emph{Proceedings of the 41st International Conference on Machine
  Learning}, volume 235 of \emph{Proceedings of Machine Learning Research},
  pp.\  34417--34431. PMLR, 21--27 Jul 2024.

\bibitem[McLeish et~al.(2024)McLeish, Bansal, Stein, Jain, Kirchenbauer,
  Bartoldson, Kailkhura, Bhatele, Geiping, Schwarzschild, and
  Goldstein]{mcleish2024transformersarithmeticrightembeddings}
Sean McLeish, Arpit Bansal, Alex Stein, Neel Jain, John Kirchenbauer, Brian~R.
  Bartoldson, Bhavya Kailkhura, Abhinav Bhatele, Jonas Geiping, Avi
  Schwarzschild, and Tom Goldstein.
\newblock Transformers can do arithmetic with the right embeddings, 2024.
\newblock URL \url{https://arxiv.org/abs/2405.17399}.

\bibitem[Meng et~al.(2022)Meng, Bau, Andonian, and Belinkov]{meng2022locating}
Kevin Meng, David Bau, Alex Andonian, and Yonatan Belinkov.
\newblock Locating and editing factual associations in gpt.
\newblock \emph{Advances in Neural Information Processing Systems},
  35:\penalty0 17359--17372, 2022.

\bibitem[Merrill \& Sabharwal(2023)Merrill and
  Sabharwal]{merrill2023expresssive}
William Merrill and Ashish Sabharwal.
\newblock The expresssive power of transformers with chain of thought.
\newblock \emph{arXiv preprint arXiv:2310.07923}, 2023.

\bibitem[Mihaylov et~al.(2018)Mihaylov, Clark, Khot, and
  Sabharwal]{mihaylov2018can}
Todor Mihaylov, Peter Clark, Tushar Khot, and Ashish Sabharwal.
\newblock Can a suit of armor conduct electricity? a new dataset for open book
  question answering.
\newblock \emph{arXiv preprint arXiv:1809.02789}, 2018.

\bibitem[Mitra et~al.(2024)Mitra, Khanpour, Rosset, and
  Awadallah]{mitra2024orcamath}
Arindam Mitra, Hamed Khanpour, Corby Rosset, and Ahmed Awadallah.
\newblock Orca-math: Unlocking the potential of slms in grade school math,
  2024.

\bibitem[Muennighoff et~al.(2024)Muennighoff, Soldaini, Groeneveld, Lo,
  Morrison, Min, Shi, Walsh, Tafjord, Lambert, Gu, Arora, Bhagia, Schwenk,
  Wadden, Wettig, Hui, Dettmers, Kiela, Farhadi, Smith, Koh, Singh, and
  Hajishirzi]{muennighoff2024olmoeopenmixtureofexpertslanguage}
Niklas Muennighoff, Luca Soldaini, Dirk Groeneveld, Kyle Lo, Jacob Morrison,
  Sewon Min, Weijia Shi, Pete Walsh, Oyvind Tafjord, Nathan Lambert, Yuling Gu,
  Shane Arora, Akshita Bhagia, Dustin Schwenk, David Wadden, Alexander Wettig,
  Binyuan Hui, Tim Dettmers, Douwe Kiela, Ali Farhadi, Noah~A. Smith, Pang~Wei
  Koh, Amanpreet Singh, and Hannaneh Hajishirzi.
\newblock Olmoe: Open mixture-of-experts language models, 2024.
\newblock URL \url{https://arxiv.org/abs/2409.02060}.

\bibitem[Nichani et~al.(2024)Nichani, Lee, and
  Bietti]{nichani2024understanding}
Eshaan Nichani, Jason~D Lee, and Alberto Bietti.
\newblock Understanding factual recall in transformers via associative
  memories.
\newblock \emph{arXiv preprint arXiv:2412.06538}, 2024.

\bibitem[NuminaMath(2024)]{numina2024winning}
NuminaMath.
\newblock How numinamath won the 1st aimo progress prize.
\newblock \url{https://huggingface.co/blog/winning-aimo-progress-prize}, 2024.

\bibitem[OpenAI(2024)]{openai2024o1}
OpenAI.
\newblock Introducing openai o1.
\newblock \url{https://openai.com/o1/}, 2024.

\bibitem[Pan et~al.(2024)Pan, Shen, Liu, Mishra, Zhang, Oliva, Raffel, and
  Panda]{pan2024dense}
Bowen Pan, Yikang Shen, Haokun Liu, Mayank Mishra, Gaoyuan Zhang, Aude Oliva,
  Colin Raffel, and Rameswar Panda.
\newblock Dense training, sparse inference: Rethinking training of
  mixture-of-experts language models.
\newblock \emph{arXiv preprint arXiv:2404.05567}, 2024.

\bibitem[Paster et~al.(2023)Paster, Santos, Azerbayev, and
  Ba]{paster2023openwebmath}
Keiran Paster, Marco~Dos Santos, Zhangir Azerbayev, and Jimmy Ba.
\newblock Openwebmath: An open dataset of high-quality mathematical web text.
\newblock \emph{arXiv preprint arXiv:2310.06786}, 2023.

\bibitem[Patel et~al.(2021)Patel, Bhattamishra, and Goyal]{patel2021nlp}
Arkil Patel, Satwik Bhattamishra, and Navin Goyal.
\newblock Are nlp models really able to solve simple math word problems?
\newblock \emph{arXiv preprint arXiv:2103.07191}, 2021.

\bibitem[Penedo et~al.(2024)Penedo, Kydl{\'\i}{\v{c}}ek, Lozhkov, Mitchell,
  Raffel, Von~Werra, Wolf, et~al.]{penedo2024fineweb}
Guilherme Penedo, Hynek Kydl{\'\i}{\v{c}}ek, Anton Lozhkov, Margaret Mitchell,
  Colin Raffel, Leandro Von~Werra, Thomas Wolf, et~al.
\newblock The fineweb datasets: Decanting the web for the finest text data at
  scale.
\newblock \emph{arXiv preprint arXiv:2406.17557}, 2024.

\bibitem[Peng et~al.(2023)Peng, Alcaide, Anthony, Albalak, Arcadinho, Biderman,
  Cao, Cheng, Chung, Grella, et~al.]{peng2023rwkv}
Bo~Peng, Eric Alcaide, Quentin Anthony, Alon Albalak, Samuel Arcadinho, Stella
  Biderman, Huanqi Cao, Xin Cheng, Michael Chung, Matteo Grella, et~al.
\newblock Rwkv: Reinventing rnns for the transformer era.
\newblock \emph{arXiv preprint arXiv:2305.13048}, 2023.

\bibitem[Petroni et~al.(2019)Petroni, Rockt{\"a}schel, Lewis, Bakhtin, Wu,
  Miller, and Riedel]{petroni2019language}
Fabio Petroni, Tim Rockt{\"a}schel, Patrick Lewis, Anton Bakhtin, Yuxiang Wu,
  Alexander~H Miller, and Sebastian Riedel.
\newblock Language models as knowledge bases?
\newblock \emph{arXiv preprint arXiv:1909.01066}, 2019.

\bibitem[Rajbhandari et~al.(2022)Rajbhandari, Li, Yao, Zhang, Aminabadi, Awan,
  Rasley, and He]{rajbhandari2022deepspeed}
Samyam Rajbhandari, Conglong Li, Zhewei Yao, Minjia Zhang, Reza~Yazdani
  Aminabadi, Ammar~Ahmad Awan, Jeff Rasley, and Yuxiong He.
\newblock Deepspeed-moe: Advancing mixture-of-experts inference and training to
  power next-generation ai scale.
\newblock In \emph{International conference on machine learning}, pp.\
  18332--18346. PMLR, 2022.

\bibitem[Roller et~al.(2021)Roller, Sukhbaatar, Weston, et~al.]{roller2021hash}
Stephen Roller, Sainbayar Sukhbaatar, Jason Weston, et~al.
\newblock Hash layers for large sparse models.
\newblock \emph{Advances in Neural Information Processing Systems},
  34:\penalty0 17555--17566, 2021.

\bibitem[Sakaguchi et~al.(2021)Sakaguchi, Bras, Bhagavatula, and
  Choi]{sakaguchi2021winogrande}
Keisuke Sakaguchi, Ronan~Le Bras, Chandra Bhagavatula, and Yejin Choi.
\newblock Winogrande: An adversarial winograd schema challenge at scale.
\newblock \emph{Communications of the ACM}, 64\penalty0 (9):\penalty0 99--106,
  2021.

\bibitem[Sanford et~al.(2024)Sanford, Fatemi, Hall, Tsitsulin, Kazemi, Halcrow,
  Perozzi, and Mirrokni]{sanford2024understanding}
Clayton Sanford, Bahare Fatemi, Ethan Hall, Anton Tsitsulin, Mehran Kazemi,
  Jonathan Halcrow, Bryan Perozzi, and Vahab Mirrokni.
\newblock Understanding transformer reasoning capabilities via graph
  algorithms.
\newblock \emph{arXiv preprint arXiv:2405.18512}, 2024.

\bibitem[Sap et~al.(2019)Sap, Rashkin, Chen, LeBras, and
  Choi]{sap2019socialiqa}
Maarten Sap, Hannah Rashkin, Derek Chen, Ronan LeBras, and Yejin Choi.
\newblock Socialiqa: Commonsense reasoning about social interactions.
\newblock \emph{arXiv preprint arXiv:1904.09728}, 2019.

\bibitem[Saxton et~al.(2019)Saxton, Grefenstette, Hill, and
  Kohli]{saxton2019analysing}
David Saxton, Edward Grefenstette, Felix Hill, and Pushmeet Kohli.
\newblock Analysing mathematical reasoning abilities of neural models.
\newblock \emph{arXiv preprint arXiv:1904.01557}, 2019.

\bibitem[Shazeer et~al.(2017)Shazeer, Mirhoseini, Maziarz, Davis, Le, Hinton,
  and Dean]{shazeer2017outrageously}
Noam Shazeer, Azalia Mirhoseini, Krzysztof Maziarz, Andy Davis, Quoc Le,
  Geoffrey Hinton, and Jeff Dean.
\newblock Outrageously large neural networks: The sparsely-gated
  mixture-of-experts layer.
\newblock \emph{arXiv preprint arXiv:1701.06538}, 2017.

\bibitem[Shi et~al.(2024)Shi, Tang, Narasimhan, and Yao]{shi2024can}
Quan Shi, Michael Tang, Karthik Narasimhan, and Shunyu Yao.
\newblock Can language models solve olympiad programming?
\newblock \emph{arXiv preprint arXiv:2404.10952}, 2024.

\bibitem[Speicher et~al.(2024)Speicher, Khan, Wu, Nanda, Das, Ghosh, Gummadi,
  and Terzi]{speicher2024understanding}
Till Speicher, Aflah~Mohammad Khan, Qinyuan Wu, Vedant Nanda, Soumi Das,
  Bishwamittra Ghosh, Krishna~P Gummadi, and Evimaria Terzi.
\newblock Understanding the mechanics and dynamics of memorisation in large
  language models: A case study with random strings.
\newblock 2024.

\bibitem[Strobl et~al.(2024)Strobl, Merrill, Weiss, Chiang, and
  Angluin]{strobl2024formal}
Lena Strobl, William Merrill, Gail Weiss, David Chiang, and Dana Angluin.
\newblock What formal languages can transformers express? a survey.
\newblock \emph{Transactions of the Association for Computational Linguistics},
  12:\penalty0 543--561, 2024.

\bibitem[Talmor \& Berant(2018)Talmor and Berant]{talmor2018web}
Alon Talmor and Jonathan Berant.
\newblock The web as a knowledge-base for answering complex questions.
\newblock \emph{arXiv preprint arXiv:1803.06643}, 2018.

\bibitem[Talmor et~al.(2018)Talmor, Herzig, Lourie, and
  Berant]{talmor2018commonsenseqa}
Alon Talmor, Jonathan Herzig, Nicholas Lourie, and Jonathan Berant.
\newblock Commonsenseqa: A question answering challenge targeting commonsense
  knowledge.
\newblock \emph{arXiv preprint arXiv:1811.00937}, 2018.

\bibitem[Tan et~al.(2024)Tan, Shen, Panda, and Courville]{tan2024scattered}
Shawn Tan, Yikang Shen, Rameswar Panda, and Aaron Courville.
\newblock Scattered mixture-of-experts implementation.
\newblock \emph{arXiv preprint arXiv:2403.08245}, 2024.

\bibitem[Tirumala et~al.(2022)Tirumala, Markosyan, Zettlemoyer, and
  Aghajanyan]{tirumala2022memorization}
Kushal Tirumala, Aram Markosyan, Luke Zettlemoyer, and Armen Aghajanyan.
\newblock Memorization without overfitting: Analyzing the training dynamics of
  large language models.
\newblock \emph{Advances in Neural Information Processing Systems},
  35:\penalty0 38274--38290, 2022.

\bibitem[Toshniwal et~al.(2024)Toshniwal, Moshkov, Narenthiran, Gitman, Jia,
  and Gitman]{toshniwal2024openmathinstruct}
Shubham Toshniwal, Ivan Moshkov, Sean Narenthiran, Daria Gitman, Fei Jia, and
  Igor Gitman.
\newblock Openmathinstruct-1: A 1.8 million math instruction tuning dataset.
\newblock \emph{arXiv preprint arXiv:2402.10176}, 2024.

\bibitem[Wang et~al.(2024)Wang, Feng, He, Tan, Han, and Tsvetkov]{wang2024can}
Heng Wang, Shangbin Feng, Tianxing He, Zhaoxuan Tan, Xiaochuang Han, and Yulia
  Tsvetkov.
\newblock Can language models solve graph problems in natural language?
\newblock \emph{Advances in Neural Information Processing Systems}, 36, 2024.

\bibitem[Wei et~al.(2022)Wei, Wang, Schuurmans, Bosma, Xia, Chi, Le, Zhou,
  et~al.]{wei2022chain}
Jason Wei, Xuezhi Wang, Dale Schuurmans, Maarten Bosma, Fei Xia, Ed~Chi, Quoc~V
  Le, Denny Zhou, et~al.
\newblock Chain-of-thought prompting elicits reasoning in large language
  models.
\newblock \emph{Advances in neural information processing systems},
  35:\penalty0 24824--24837, 2022.

\bibitem[Weiss et~al.(2021)Weiss, Goldberg, and
  Yahav]{DBLP:conf/icml/WeissGY21}
Gail Weiss, Yoav Goldberg, and Eran Yahav.
\newblock Thinking like transformers.
\newblock In Marina Meila and Tong Zhang (eds.), \emph{Proceedings of the 38th
  International Conference on Machine Learning, {ICML} 2021, 18-24 July 2021,
  Virtual Event}, volume 139 of \emph{Proceedings of Machine Learning
  Research}, pp.\  11080--11090. {PMLR}, 2021.
\newblock URL \url{http://proceedings.mlr.press/v139/weiss21a.html}.

\bibitem[Welbl et~al.(2017)Welbl, Liu, and Gardner]{welbl2017crowdsourcing}
Johannes Welbl, Nelson~F Liu, and Matt Gardner.
\newblock Crowdsourcing multiple choice science questions.
\newblock \emph{arXiv preprint arXiv:1707.06209}, 2017.

\bibitem[Yang et~al.(2024)Yang, Yang, Hui, Zheng, Yu, Zhou, Li, Li, Liu, Huang,
  et~al.]{yang2024qwen2}
An~Yang, Baosong Yang, Binyuan Hui, Bo~Zheng, Bowen Yu, Chang Zhou, Chengpeng
  Li, Chengyuan Li, Dayiheng Liu, Fei Huang, et~al.
\newblock Qwen2 technical report.
\newblock \emph{arXiv preprint arXiv:2407.10671}, 2024.

\bibitem[Yang et~al.(2018)Yang, Qi, Zhang, Bengio, Cohen, Salakhutdinov, and
  Manning]{yang2018hotpotqa}
Zhilin Yang, Peng Qi, Saizheng Zhang, Yoshua Bengio, William~W Cohen, Ruslan
  Salakhutdinov, and Christopher~D Manning.
\newblock Hotpotqa: A dataset for diverse, explainable multi-hop question
  answering.
\newblock \emph{arXiv preprint arXiv:1809.09600}, 2018.

\bibitem[Yu et~al.(2023)Yu, Jiang, Shi, Yu, Liu, Zhang, Kwok, Li, Weller, and
  Liu]{yu2023metamath}
Longhui Yu, Weisen Jiang, Han Shi, Jincheng Yu, Zhengying Liu, Yu~Zhang,
  James~T Kwok, Zhenguo Li, Adrian Weller, and Weiyang Liu.
\newblock Metamath: Bootstrap your own mathematical questions for large
  language models.
\newblock \emph{arXiv preprint arXiv:2309.12284}, 2023.

\bibitem[Yue et~al.(2024)Yue, Zheng, Zhang, and Chen]{yue2024mammoth2}
Xiang Yue, Tuney Zheng, Ge~Zhang, and Wenhu Chen.
\newblock Mammoth2: Scaling instructions from the web.
\newblock \emph{arXiv preprint arXiv:2405.03548}, 2024.

\bibitem[Zellers et~al.(2019)Zellers, Holtzman, Bisk, Farhadi, and
  Choi]{zellers2019hellaswag}
Rowan Zellers, Ari Holtzman, Yonatan Bisk, Ali Farhadi, and Yejin Choi.
\newblock Hellaswag: Can a machine really finish your sentence?
\newblock \emph{arXiv preprint arXiv:1905.07830}, 2019.

\bibitem[Zhang et~al.(2022)Zhang, Backurs, Bubeck, Eldan, Gunasekar, and
  Wagner]{zhang2022unveiling}
Yi~Zhang, Arturs Backurs, S{\'e}bastien Bubeck, Ronen Eldan, Suriya Gunasekar,
  and Tal Wagner.
\newblock Unveiling transformers with lego: a synthetic reasoning task.
\newblock \emph{arXiv preprint arXiv:2206.04301}, 2022.

\bibitem[Zhang(2024)]{stackmathqa2024}
Yifan Zhang.
\newblock Stackmathqa: A curated collection of 2 million mathematical questions
  and answers sourced from stack exchange, 2024.
\newblock URL \url{https://huggingface.co/datasets/math-ai/StackMathQA}.

\bibitem[Zhao et~al.(2023)Zhao, Gu, Varma, Luo, Huang, Xu, Wright, Shojanazeri,
  Ott, Shleifer, et~al.]{zhao2023pytorch}
Yanli Zhao, Andrew Gu, Rohan Varma, Liang Luo, Chien-Chin Huang, Min Xu, Less
  Wright, Hamid Shojanazeri, Myle Ott, Sam Shleifer, et~al.
\newblock Pytorch fsdp: experiences on scaling fully sharded data parallel.
\newblock \emph{arXiv preprint arXiv:2304.11277}, 2023.

\bibitem[Zhao et~al.(2024)Zhao, Lee, and Hsu]{zhao2024large}
Zirui Zhao, Wee~Sun Lee, and David Hsu.
\newblock Large language models as commonsense knowledge for large-scale task
  planning.
\newblock \emph{Advances in Neural Information Processing Systems}, 36, 2024.

\bibitem[Zhong et~al.(2024)Zhong, Xia, Chen, and Lewis]{zhong2024lory}
Zexuan Zhong, Mengzhou Xia, Danqi Chen, and Mike Lewis.
\newblock Lory: Fully differentiable mixture-of-experts for autoregressive
  language model pre-training.
\newblock \emph{arXiv preprint arXiv:2405.03133}, 2024.

\bibitem[Zhou et~al.(2023)Zhou, Bradley, Littwin, Razin, Saremi, Susskind,
  Bengio, and Nakkiran]{zhou2023algorithmstransformerslearnstudy}
Hattie Zhou, Arwen Bradley, Etai Littwin, Noam Razin, Omid Saremi, Josh
  Susskind, Samy Bengio, and Preetum Nakkiran.
\newblock What algorithms can transformers learn? a study in length
  generalization, 2023.
\newblock URL \url{https://arxiv.org/abs/2310.16028}.

\bibitem[Zhou et~al.(2022)Zhou, Lei, Liu, Du, Huang, Zhao, Dai, Le, Laudon,
  et~al.]{zhou2022mixture}
Yanqi Zhou, Tao Lei, Hanxiao Liu, Nan Du, Yanping Huang, Vincent Zhao, Andrew~M
  Dai, Quoc~V Le, James Laudon, et~al.
\newblock Mixture-of-experts with expert choice routing.
\newblock \emph{Advances in Neural Information Processing Systems},
  35:\penalty0 7103--7114, 2022.

\bibitem[Zhou et~al.(2024)Zhou, Alon, Chen, Wang, Agarwal, and
  Zhou]{zhou2024transformers}
Yongchao Zhou, Uri Alon, Xinyun Chen, Xuezhi Wang, Rishabh Agarwal, and Denny
  Zhou.
\newblock Transformers can achieve length generalization but not robustly.
\newblock \emph{arXiv preprint arXiv:2402.09371}, 2024.

\bibitem[Zhu et~al.(2024)Zhu, Guo, Shao, Yang, Wang, Xu, Wu, Li, Gao, Ma,
  et~al.]{zhu2024deepseek}
Qihao Zhu, Daya Guo, Zhihong Shao, Dejian Yang, Peiyi Wang, Runxin Xu, Y~Wu,
  Yukun Li, Huazuo Gao, Shirong Ma, et~al.
\newblock Deepseek-coder-v2: Breaking the barrier of closed-source models in
  code intelligence.
\newblock \emph{arXiv preprint arXiv:2406.11931}, 2024.

\bibitem[Zoph et~al.(2022)Zoph, Bello, Kumar, Du, Huang, Dean, Shazeer, and
  Fedus]{zoph2022st}
Barret Zoph, Irwan Bello, Sameer Kumar, Nan Du, Yanping Huang, Jeff Dean, Noam
  Shazeer, and William Fedus.
\newblock St-moe: Designing stable and transferable sparse expert models.
\newblock \emph{arXiv preprint arXiv:2202.08906}, 2022.

\end{thebibliography}
